%% file: main.tex
\documentclass{article}
\usepackage{graphicx} %

\input{importsanddeclares}

\newcommand{\fmlp}{f_{\mathrm{MLP}}}

\newcommand{\fmoe}{f_{\mathrm{MoE}}}

\newcommand{\topk}{\mathrm{top}k}

\usepackage{tikz}
\usetikzlibrary{positioning, arrows.meta}

\title{Secret mixtures of experts inside your LLM}

\author{Enric Boix-Adsera\thanks{Wharton School of Statistics and Data Science at the University of Pennsylvania, \texttt{eboix@wharton.upenn.edu}}}

\begin{document}

\maketitle

\begin{abstract}

Despite being one of the earliest neural network layers, the Multilayer Perceptron (MLP) is arguably one of the least understood parts of the transformer architecture due to its dense computation and lack of easy visualization. This paper seeks to understand the MLP layers in dense LLM models by hypothesizing that these layers secretly approximately perform a sparse computation -- namely, that they can be well approximated by sparsely-activating Mixture of Experts (MoE) layers.

Our hypothesis is based on a novel theoretical connection between MoE models and Sparse Autoencoder (SAE) structure in activation space. We empirically validate the hypothesis on pretrained LLMs, and demonstrate that the activation distribution matters -- these results do not hold for Gaussian data, but rather rely crucially on structure in the distribution of neural network activations.

Our results shine light on a general principle at play in MLP layers inside LLMs, and give an explanation for the effectiveness of modern MoE-based transformers. Additionally, our experimental explorations suggest new directions for more efficient MoE architecture design based on low-rank routers.

\end{abstract}

\section{Introduction}

Despite being one of the earliest neural network modules, the Multilayer Perceptron (MLP) arguably remains one of the least understood parts of the transformer architecture. Unlike attention mechanisms, whose attention patterns can be visualized \cite{vaswani2017attention}, MLP layers resist straightforward inspection. How can we understand what MLP layers are actually doing inside a trained transformer?

A popular approach to study MLPs is to provide a mathematical analysis proceeding from simplifying assumptions. These assumptions could be stylized hyperparameters (e.g. \cite{jacot2018neural,li2020towards,jacot2021saddle,allen2019convergence}), toy data distributions (e.g., \cite{abbe2023sgd,boix2023can,nichani2024transformers}), or simplified architectures (e.g., \cite{saxe2013exact,arora2018convergence,bartlett2018gradient,radhakrishnan2024mechanism,zhu2025iteratively}). While such theoretical analyses can yield valuable insights, their simplifications raise a critical question: how much do they reflect models trained in practice? And how much can they directly tell us about the function of the MLP layer inside a real-world, trained transformer? More broadly:

\begin{center}
\textit{Is it possible to understand what MLP layers do while avoiding starting from simplifying assumptions, such as on (1) the hyperparameter regime, (2) the data distribution, or (3) the architecture?}
\end{center}

\paragraph{Our contributions} In this paper, rather than starting from a toy simplification and analyzing it, we seek to understand MLP layers by starting from a scientific hypothesis on the type of function that they secretly compute inside of trained transformers. Then, we experimentally check the extent to which this hypothesis is valid.

\begin{figure}[h]
\centering
\includegraphics[width=0.8\linewidth,trim={0 7cm 0 14cm},clip]{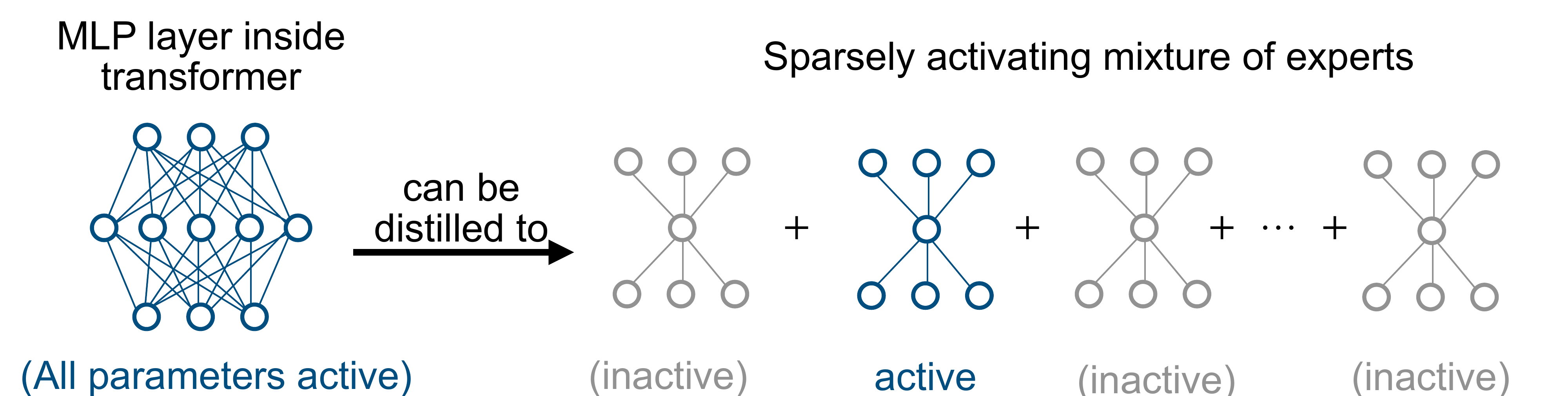}
\caption{In this paper we hypothesize, and then experimentally validate, that an MLP layer in the middle of a pretrained transformer can be effectively described by a sparsely-activating mixture-of-experts layer.}\label{fig:hypothesis}
\end{figure}

\begin{enumerate}
\item \textbf{Hypothesis: MLPs in trained transformers secretly implement sparse mixtures of experts}.

We hypothesize that MLP layers in trained transformers have a certain hidden structure: they can be well approximated by sparsely-activating Mixture of Experts layers (MoE) \cite{jacobs1991adaptive,eigen2013learning}, with a much smaller number of active parameters than the original MLP; see Figure~\ref{fig:hypothesis}.

This hypothesis is spurred by a novel \textbf{mathematical insight showing a connection between sparse autoencoder structure in activation space and secret Mixture of Experts structure}. Namely, it has been observed that activations in language models are typically sparse in some dictionary \cite{bricken2023monosemanticity}. Starting from this observation, we prove theorems that suggest that the MLPs in trained networks should also be well approximated by sparsely-activating MoE layers.

We also prove that the dictionary-sparsity in the input distribution and the model architecture is critical, since in contrast for Gaussian inputs (which do not have dictionary-sparse input structure), we prove that MLPs should generally not be approximable by sparsely-activating MoEs.

\item \textbf{Empirical validation of hypothesis: MLP layers in trained models have secret MoE structure}. Next, we empirically validate the hypothesis. By distilling from a dense MLP in a pretrained model to a sparsely-activating MoE, we show that layers in pretrained LLMs can be approximated well by sparse MoEs. We compare this to the ``control'' condition, where the input distribution is Gaussian distribution (with matched covariance), in which case we demonstrate that MoEs are not a good approximation. 

As a bonus, our distillation experiments yield a promising direction for efficient MoE architecture design -- we find that using a low-rank router improves overall distillation performance while reducing computational cost.

\end{enumerate}

\begin{figure}
\begin{tabular}{cc}
\includegraphics[width=0.45\textwidth]{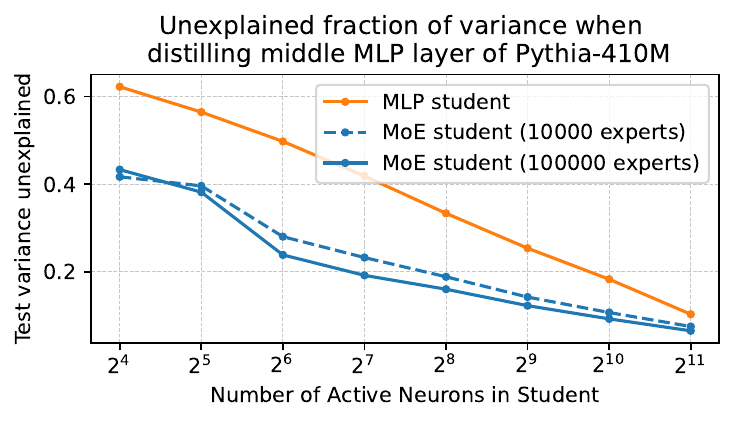} & \includegraphics[width=0.45\textwidth]{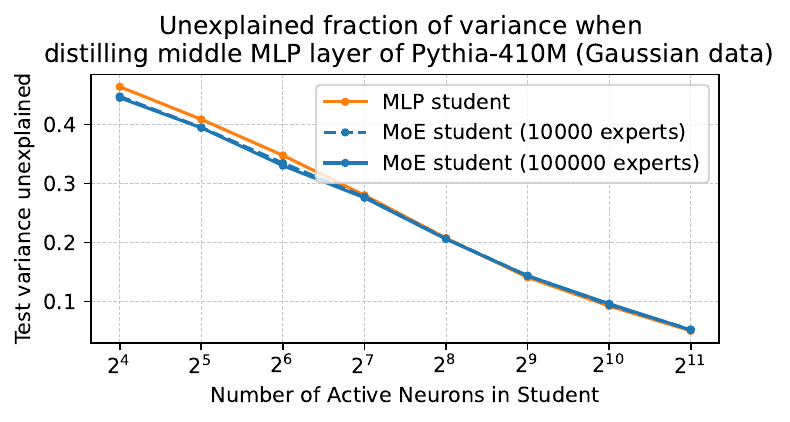}
\end{tabular}
\caption{We distill the middle MLP layer of Pythia-410M to either a smaller MLP student model, or an MoE student model with fewer active parameters. On the left, we see that under the input distribution induced by the previous layers, MoE students can achieve the same distillation performance with fewer active parameters than MLPs. On the right, under a Gaussian input distribution with the same mean and covariance, MoE students yield no significant gain, showing that the data distribution is crucial for the secret MoE structure. See Section~\ref{sec:experimental-validation} and Appendix~\ref{app:additional-experiments} for details and further experiments.}\label{fig:teaser-experiments}
\end{figure}

The organization of the rest of this paper is as follows. Section~\ref{sec:preliminaries} presents preliminaries, including the definitions of MLP and MoE architectures. Section~\ref{sec:hypothesis} presents the mathematical analysis that spurs the secret mixture of experts hypothesis. Finally, Section~\ref{sec:experimental-validation} presents experimental validation of the hypothesis on several pretrained language models, by distilling MLP layers to more structured sparse MLP layers. We discuss the broader implications of our approach as a framework for conducting research in architecture design and studying pretrained models in Section~\ref{sec:discussion}. We discuss further related work in Section~\ref{sec:further-related}.

\section{Preliminaries}\label{sec:preliminaries}

\subsection{MLP and MoE architectures} 
We consider inputs $x \in \mathbb{R}^d$, drawn from some distribution $D$.  The MLP module (also called feedforward module) is the earliest designed neural network \citep{rosenblatt1958perceptron,rosenblatt1962principles}, and has the form:
\begin{align}
\fmlp(x;A,B) = A\sigma(Bx)\mbox{ where } A \in \R^{d_{out} \times d_{mlp}}, B \in \mathbb{R}^{d_{mlp} \times d}\,,
\end{align}
where $\sigma : \R \to \R$ is an activation function applied elementwise, $d_{mlp}$ is the intermediate dimension, and $d_{out}$ is the output dimension. Standard transformer architectures have alternating layers of MLPs and attention modules \cite{vaswani2017attention}.

The MoE module \citep{jacobs1991adaptive,eigen2013learning,shazeer2017outrageously} with $m$ experts has a gating function $g : \R^d \to \R^m$ and parameters $A_1,\ldots,A_m \in \R^{d_{out} \times d_{exp}}$, and $B_1,\ldots,B_m \in \R^{d_{exp} \times d}$, and has the form:
\begin{align}\label{eq:moe-definition}
\fmoe(x;\{A_i,B_i\}_{i \in [m]}, g) = \sum_{i=1}^m g(x)_i \fmlp(x;A_i,B_i)\,.
\end{align}

When the gating function has sparsity $\|g(x)\|_0 \leq k$ for all inputs, the layer is a $(m,k)$-MoE, meaning that only $k$ out of the $m$ experts are active. Since only the active experts have to be evaluated, this results in significant computational efficiency in training and in inference.

The most popular gating function is ``linear routing'' with top-$k$ activation, which is what we consider in this paper. For completeness, we define this gating function below.

\begin{definition} Let $\topk : \R^m \to (\{-\infty\} \cup\R)^m$ be the map that preserves the top $k$ entries, sending all other entries to $-\infty$. Let $\beta \geq 0$ be the inverse temperature parameter. Define $\smax : \R^m \to \R^m$ as
\begin{align*}
\smax(z_1,\ldots,z_m) = [\frac{e^{z_1}}{\sum_{i=1}^m e^{z_i}}, \ldots, \frac{e^{z_m}}{\sum_{i=1}^m e^{z_i}}]\,.
\end{align*}
Then, linear routing has parameters $\beta \geq 0$ and a matrix $R \in \R^{m \times d}$, and is given by\footnote{Out of convenience we take the convention that $0 \times -\infty = -\infty$ so that the gating is continuous in $\beta$.}
\begin{align*}
g(x;\beta, R) := \smax(\beta \cdot \topk(Rx))\,.
\end{align*}
\end{definition}

\begin{definition} We say that $g$ is a ``hard'' gating function if it places uniform weight on the active experts: namely $g(x) \in \{0,1/\|g(x)\|_0\}^m$ for any input $x$. Notice that, when $\beta = 0$, then linear routing is a hard gating function.
\end{definition}

\section{Hypothesized structure: secret mixtures of experts in your model}\label{sec:hypothesis}

The main hypothesis of this paper is:
\begin{hypothesis}\label{hyp:main-hypothesis}
In a pretrained transformer model, the dense MLP layers can be well-approximated by sparsely-activating MoE layers.
\end{hypothesis}

We emphasize that this hypothesis is \textit{not} that a transformer architecture with MoE layers instead of MLP layers will have good performance. Indeed, it is known that transformers with MoE layers perform well \cite{shazeer2017outrageously}, and these are becoming an increasingly common architectural choice in frontier language models \citep{liu2024deepseek,meta2025llama4,jiang2024mixtral,mosaic2024introducing,qwen2024qwen,snowflake2024snowflake,agarwal2025gpt}.

Instead, our hypothesis is that, in dense transformer models \textit{with MLP layers}, the MLP layers can be well-approximated by sparsely-activating MoE layers. If the hypothesis is true, then it helps explain why replacing these layers with MoE layers is a good choice to preserve performance with a smaller number of active parameters.

We formally define below what is meant by approximating a model under an input distribution $D$.
\begin{definition}
A function $f^*(x)$ is $\epsilon$-\textit{approximable} by an $(m,k)$-MoE model with gating function $g$ and expert dimension $d_{exp}$ if there are $A_1,\ldots,A_m \in \R^{d_{out} \times d_{exp}}$ and $B_1,\ldots,B_m \in \R^{d_{exp} \times d}$ such that 
\begin{align*}
\E_{x \sim D}[\|f^*(x) - \fmoe(x;\{A_i,B_i\}_{i \in [m]}, g)\|^2] \leq \epsilon\,.
\end{align*}
\end{definition}

In the context of our hypothesis, the input distribution $D$ is the distribution induced by pushing forward the input distribution to the network (such as a text distribution) through the layers preceding the MLP layer. In other words, if we are considering approximating the MLP at layer $\ell$, then $D$ is the distribution of internal activations that are inputted to the MLP at layer $\ell$.

\subsection{Under Gaussian input distribution, dense MLPs cannot be approximated by sparsely-activating MoEs}

Hypothesis~\ref{hyp:main-hypothesis} is a significant claim on the structure of trained models. Indeed, we prove below that this hypothesis is implausible if the input data to the MLP is isotropic Gaussian (which is a common assumption in deep learning theory). We prove that the identity function (one of the simplest functions expressible by a dense MLP) cannot be approximated by sparse MoEs if the data is Gaussian.

\begin{theorem}[Inapproximability of identity by sparse MoEs under Gaussian data distribution]\label{thm:inapproximability-moe-gaussian}
There are universal constants $c,c' > 0$ such that the following is true. Under isotropic Gaussian input distribution $D = N(0,I_d/d)$, \textbf{there is no} $(m,k)$-MoE with hard gating function and a number of active neurons $kd_{exp} < d/2$ and number of expert configurations $m^k \leq \exp(cd)$, that can $c'$-approximate the identity function $f^*(x) = x$.
\end{theorem}

\begin{proof}[Proof sketch]
See Appendix~\ref{app:gaussian-distribution-identity-inapproximable} for the full proof. We outline the main ideas below.
The mixture-of-expert's gating function splits the input space $\R^d$ into a collection of measurable regions $\{U_S \subseteq \R^d \}_{S \subseteq [m]}$, each of which has a different subset $S$ of experts active. Formally, $U_S = \{x \in \R^d : \mbox{support}(g(x)) = S\}$ where $g$ is the gating function.

We can restrict to considering regions $U_S$ that have lower-bounded probability
\begin{align}\P[x \in U_S] \geq 1/(3(m+1)^k),\label{ineq:region-probability-lower-bound}
\end{align}
since by a union bound argument the regions with smaller probability provide negligible contributions and can be ignored in our analysis.

Let us now argue that the mixture of experts model $\fmoe$ cannot approximate the identity function $f^*(x) = x$ on any region $U_S$ satisfying the probability lower bound \eqref{ineq:region-probability-lower-bound}. On inputs restricted to region $U_S$, the mixture of experts $\fmoe$ is a sum of $k$ MLPs of width $d_{exp}$. In other words, on this region the mixture of experts is given by an MLP of width $kd_{exp}$. Since an MLP depends on a subspace of the inputs of dimensionality at most equal to its width, there is a linear projection $\Pi_S \in \R^{kd_{exp} \times d}$ and a function $f_S : \R^{kd_{exp}} \to \R^d$ such that
\begin{align*}
\fmoe(x) = f_S(\Pi_S x) \mbox{ for all } x \in U_S\,.
\end{align*}
Intuitively, since $kd_{exp} < d/2$, the projection $\Pi_S$ to a lower-dimensional space ``loses information'' about $x$. Therefore the mixture-of-experts should not be able to compute the identity function $f^*(x)$ on this region. The only way in which the identity function can be computed is if  the region is degenerate -- i.e.,  $U_S$ mostly lies in a low-dimensional subspace of the input space. However, this would imply that the region has small probability mass, contradicting the probability mass lower bound condition \eqref{ineq:region-probability-lower-bound}.

We make this intuition precise by leveraging a technical lemma of \cite{boix2025power}, which was developed for a different purpose (studying the expressive power of granularity in mixture of experts models). The full proof is in Appendix~\ref{app:gaussian-distribution-identity-inapproximable}.

\end{proof}

\subsection{Nevertheless, dictionary-sparse structure in inputs implies approximability of MLPs by MoEs}

The result of the previous section indicates that for our Hypothesis~\ref{hyp:main-hypothesis} to be true, then there must be extra structure in the input distribution, which allows the MLP's computations to be well approximated by a mixture of experts. Indeed, our the hypothesis is motivated by a better understanding of the structure of the input distribution, beyond the crude assumption of Gaussianity.

We proceed from the observation of \cite{bricken2023monosemanticity} that neural network activations are approximately sparse in some dictionary of vectors. Let us formalize this observation.

\begin{definition}[Dictionary-sparse structure]
Distribution $D$ over vectors is $(m,k)$-dictionary-sparse if there is a dictionary of vectors $v_1,\ldots,v_m$ such that any vector $x$ in the support of $D$ lies in $x \in \mathrm{span}\{v_{i_1},\ldots,v_{i_k}\}$ for some $i_1,\ldots,i_k \in [m]$.\footnote{In practice, the activation distribution is not perfectly sparse in the dictionary, but only approximately. Our results below can be readily adapted with an extra additive error term to account for this approximation, but we omit this to keep the notation and discussion simple.}
\end{definition}

For the purposes of our analysis, we additionally posit the property the dictionary's vectors are approximately orthogonal to each other. This is a natural property to expect in high dimensions, since the normalized inner product of two random vectors in dimension $d$ is $O(1/\sqrt{d})$. Additionally, approximate orthogonality has been argued to be key to how networks represent concepts in superposition \cite{elhage2022toy}:
\begin{definition}[Approximately orthogonal dictionary]
A dictionary $v_1,\ldots,v_m \in \R^d$ is $(\gamma,k)$-orthogonal if 
for any $x \in \R^d$ and $i_1,\ldots,i_k \in [m]$ we have that 
$\|\sum_{j=1}^k \hat{v}_{i_j} \hat{v}_{i_j}^{\top} x - z\| \leq \gamma \|z\|$, where $z$ is the projection of $x$ to the span of $v_{i_1},\ldots,v_{i_k}$.
\end{definition}

Now we show that Hypothesis~\ref{hyp:main-hypothesis} is reasonable when the input distribution satisfies the above properties. As long as the input distribution is dictionary-sparse with an approximately-orthogonal dictionary, we show that sparse MoEs can represent any linear function. Thus, they overcome the obstacle from Gaussian input distribution, where they could not even represent the special case of an identity function $f^*(x) = x$.

\begin{theorem}[Linear functions are approximable by sparse MoEs under sparse-dictionary data]\label{thm:approx-linear-with-sparse-dictionary-data}
Suppose that the data distribution $D$ is supported in the unit ball, and is an $(m,k)$-dictionary-sparse distribution for a $(\gamma,k)$-approximately-orthogonal dictionary.

Then for any linear function $f^*(x) = Ax$ there is a hard-gated $(m,k)$-MoE with $d_{expert} = 1$ that $\gamma \|A\|_{\mathrm{op}}$-approximates $f^*$.
\end{theorem}
\begin{proof}
Let the MoE have $m$ experts, each corresponding to one of the elements of the dictionary. Let the $i$th expert compute $f_i(x) = k A \hat{v}_i \hat{v}_i^{\top} x$, which is implementable with a single neuron and linear activation function.

Next, let the hard gating function be such that, on input $x$, the gating function activates $k$ experts $i_1,\ldots,i_k$ such that $x \in \mathrm{span}\{v_{i_1},\ldots,v_{i_k}\}$. 

Notice $\fmoe(x) = \sum_{j=1}^k A \hat{v}_{i_j} \hat{v}_{i_j}^{\top}$. Since $x \in \mathrm{span}\{v_{i_1},\ldots,v_{i_k}\}$, by the $(\gamma,k)$-approximate-orthogonality property it follows that $\|\fmoe(x) - Ax\| \leq \|A\|_{\mathrm{op}} \|\fmoe(x) - x\| \leq \gamma \|A\|_{\mathrm{op}} \|x\| \leq \gamma \|A\|_{\mathrm{op}}$.
\end{proof}

A few remarks to help interpret this theorem are in order.

\begin{remark}[Making sense of the number of active parameters]
The number of active parameters of the MoE should be contrasted to the number of active parameters that would be needed to perform the same approximation with an MLP. For simplicity, let us consider the identity function $f^*(x) = x$ as our linear function, and let us consider a dictionary $e_1,\ldots,e_d$ of the standard basis vectors and a distribution $D$ which is uniform on $\{e_1,\ldots,e_d\}$. In this setting, the above theorem guarantees that there is a mixture of $d$ single-neuron experts, of which exactly $1$ is active on any input, which computes $f^*$ perfectly. On the other hand, in order to obtain this with an MLP, the output of the MLP has to be able to span the full $d$-dimensional space, which means that it must have at least $d$ neurons. \textbf{Therefore, the theorem shows a factor of $d$ decrease in the number of active expert parameters with a sparse MoE over a dense MLP.}
\end{remark}

\begin{remark}[On the approximation guarantee]
In $d$ dimensions, we expect the approximate orthogonality of our dictionary to be on the order of $\gamma = O(1 / \sqrt{d})$, since this is the approximate magnitude of the inner product of two random vectors on the sphere. So, as the dimension increases, the error bound of the theorem should tend to 0.
\end{remark}

\begin{remark}[On the implementation of the gating scheme]
Although we are not describing how to compute the gating function in the above theorem, in practice \cite{bricken2023monosemanticity,bussmann2024batchtopk,gao2024scaling} have shown that  the active dictionary indices $i_1,\ldots,i_k$ in a sparse autoencoder can be found with a linear projection followed by a $\topk$ operation. In these cases, the MoE gating scheme can be implemented with linear routing and $\beta = 0$.
\end{remark}

In Appendix~\ref{app:moe-nonlinear-extension} we extend this Theorem~\ref{thm:approx-linear-with-sparse-dictionary-data} to nonlinear functions. We state our result informally below.
\begin{informaltheorem}[Nonlinear functions are approximable by sparse MoEs under dictionary data]\label{thm:approx-nonlinear-with-sparse-dictionary-data}
Suppose that the distribution $D$ is supported in the unit ball, and is an $(m,k)$-dictionary-sparse distribution for a $(\gamma,k)$-approximately-orthogonal dictionary.

If $f^* : \R^d \to \R^d$ is a homogeneous polynomial given by a tensor with ``rank-$r$ interactions between features of the dictionary'', then there is a hard-gated $(m,k)$-MoE with $d_{expert} \leq O_p(r)$ that $\gamma \|A\|_{\mathrm{op}}$-approximates $f^*$.
\end{informaltheorem}
See Appendix~\ref{app:moe-nonlinear-extension} for the relevant definitions and a formal statement and proof of this extension.

\section{Experiment: distillation uncovers secret MoE structure}\label{sec:experimental-validation}

\begin{figure}
\centering
\includegraphics[width=0.6\linewidth,trim={1cm 17cm 25cm 0cm},clip]{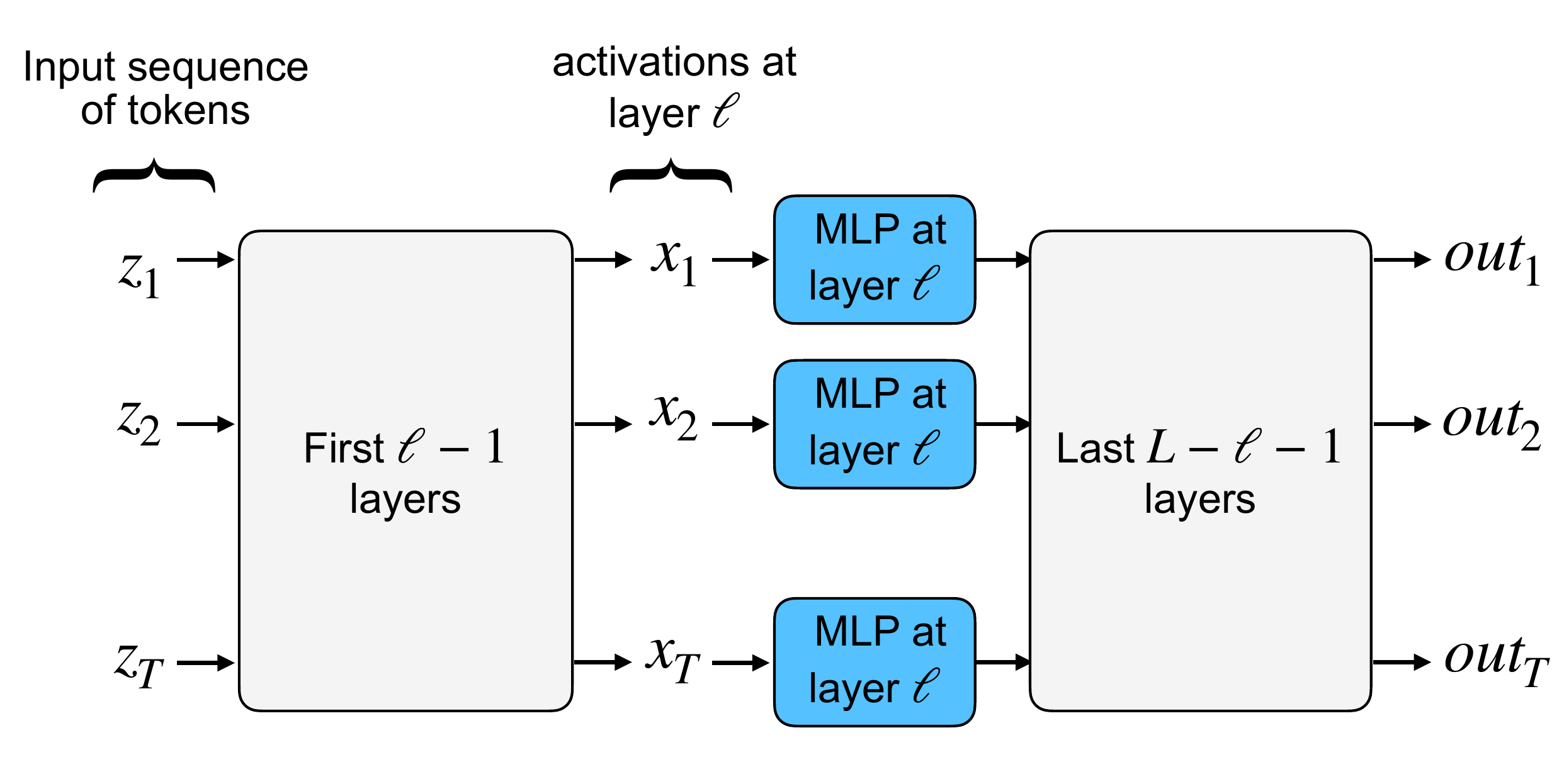}
\caption{The dataset of internal activations is created by pushing forward datasets of text through all layers preceding the MLP that we seek to distill.}\label{fig:pushforwardactivations}
\end{figure}

We empirically test the Secret MoE Hypothesis (Hypothesis~\ref{hyp:main-hypothesis}) on pretrained transformer LLMs. Our experiments report the loss of distilling a dense MLP at some layer $\ell$ inside of a pretrained language model to a sparsely-activating mixture of experts.  We distill the middle MLP layers of three dense transformer architectures of increasing scale: Pythia-70m (layer $\ell = 3$), Gemma-270m (layer $\ell = 9$), and Pythia-410m (layer $\ell = 12$).

\paragraph{Datasets over which we distill} The datasets that we distill over are generated as follows.
We start with a dataset $\cD^{text}$ of tokenized texts $\{(z_1^{(i)},\ldots,z_{T_i}^{(i)})\}_{i \in [n]}$ of varying lengths that are inputted to the LLM. We push each text $(z_1^{(i)},\ldots,z_{T_i}^{(i)})$ forward through the first $\ell-1$ layers of the network to get a sequence of activations $(x_1^{(i)},\ldots,x_{T_i}^{(i)})$ that are the inputs to the MLP at layer $\ell$ (see Figure~\ref{fig:pushforwardactivations}). We form the dataset $\cD^{act}$ by concatenating all activations generated in this way $$\cD^{act} = \begin{bmatrix} x_1^{(1)},\ldots,x_{T_1}^{(1)},\ldots, x_1^{(i)}, \ldots,x_{T_i}^{(i)}, \ldots, x_1^{(n)},\ldots, x_{T_n}^{(n)} \end{bmatrix} \in \R^{d \times (T_1 + \dots T_n)}.$$ We run this procedure to generate a training dataset $\cD^{act,train}$ and a testing dataset $\cD^{act,test}$ from corresponding train and test splits of the text datasets. We use Wikitext-103 \cite{merity2016pointer} in our experiments, filtering out texts with fewer than 20 tokens to select higher-quality text data. This yields a total of around 4M training samples and 200K test samples, which is a computationally manageable quantity for distillation given our computational budget.

We additionally generate datasets $\cD^{gauss,train}$ and $\cD^{gauss,test}$ of Gaussian data with the same mean and covariance as $\cD^{act,train}$ and the same corresponding numbers of train samples and test samples.

\begin{figure}[h]

\centering

\begin{tabular}{c|c}
\textbf{Activation Data $\cD^{act}$} & \begin{tabular}{c}\textbf{Gaussian data $\cD^{gauss}$} \\ {\textbf{with 
matching mean and covariance}}\end{tabular} \\
\hline
\includegraphics[width=0.45\textwidth]{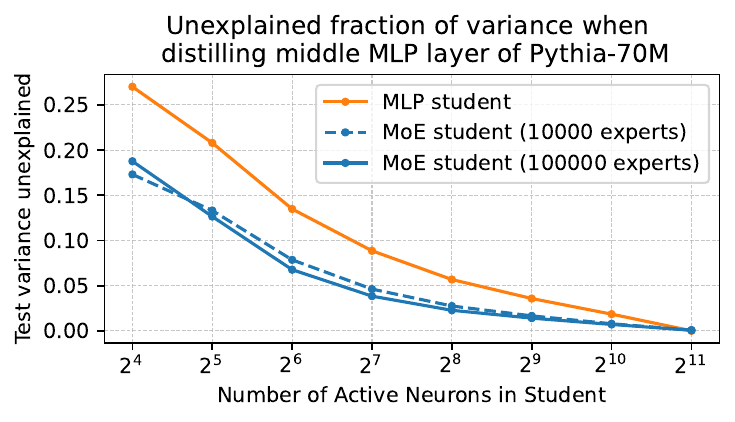} & \includegraphics[width=0.45\textwidth]{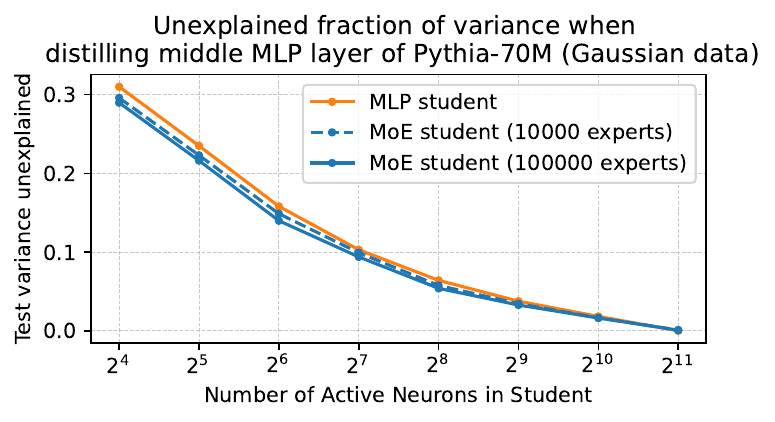} \\
\includegraphics[width=0.45\textwidth]{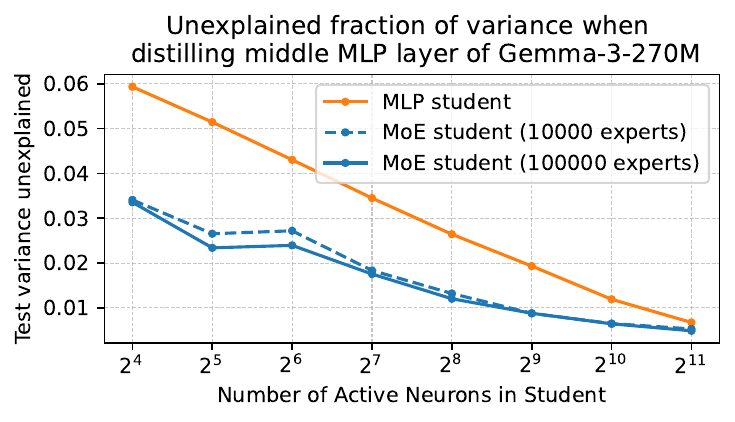} & \includegraphics[width=0.45\textwidth]{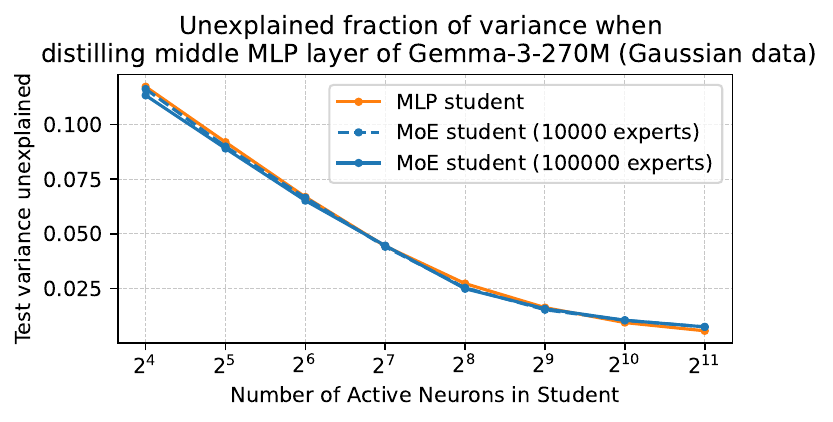}\end{tabular}
\caption{The unexplained fraction of the variance in the outputs from distilling the middle MLP layer of Pythia-70M (first row), Gemma-270M (second row). Results for Pythia-410M are in Figure~\ref{fig:teaser-experiments}. In the left column, we observe that over the activation dataset $\cD^{act}$ sparse MoE students are able to capture a significantly higher amount of the variance than corresponding MLP students with the same number of active neurons. In particular, for Pythia-410M and Gemma-3-270M there are cases in which the sparse MoE captures the same variance as the MLP using 8 times fewer active neurons. On the other hand, the distillation results in the right column demonstrate that MoE students have little advantage when the data distribution is instead Gaussian (with matched mean and covariance). }\label{fig:main-distillation-results}

\end{figure}

\paragraph{Student-teacher setup} We fit a student MoE model in a student-teacher setup, where the teacher is the pretrained model's MLP layer $\ell$. We distill over both input datasets $\cD^{act,train}$ $\cD^{gauss,train}$ and compare performance. For each student-teacher-dataset tuple we train with Adam for $100$ epochs with mean-squared error and batch size 1024. We sweep over the learning rate hyperparameter in the range 1e-3, 3e-4, and 1e-4, and choose the one with the best final test loss.

We achieve the best distillation performance by distilling to a student MoE with a shared expert. This is a less expressive variant of the MoE defined in \eqref{eq:moe-definition}, but it is used in practice in leading open-source LLMs since it is easier to optimize \cite{liu2024deepseek,team2024qwen2}. The MoE with shared expert has the form:
\begin{align}\label{eq:shared-expert-moe}
f_{MoE+Shared}(x; A,B, \{C_i,D_i\}_{i \in [m]}, g) = f_{MLP}(x;A,B) + f_{MoE}(x;\{C_i,D_i\}_{i \in [m]}, g)\,.
\end{align}
In our experiments, our MoE has single-neuron experts $d_{exp} = 1$, and we pick the inner dimension of the shared MLP to equal the total number of active experts $k$. Therefore, this architecture has $d_{mlp} + kd_{exp} = 2k$ active neurons, and it is strictly less expressive than a pure MoE architecture as in \eqref{eq:moe-definition} with $2k$ active experts. Nevertheless, in our setting the shared expert yields improved performance, likely due to a better optimization landscape (see Appendix~\ref{app:shared-expert-ablation} for an ablation experiment).

Additionally, in order to reduce computational costs we reparametrize the linear routing matrix in the gating function as
\begin{align*}
R = R_1 R_2\,,
\end{align*}
with trainable matrices $R_1 \in \R^{m \times d_{proj}}$ and $R_2 \in R^{d_{proj} \times d}$. Here, $d_{proj}$ is a smaller inner projection dimension  than the outer dimensions of the routing matrix. Surprisingly, we find that although this reparametrization makes the linear router strictly less expressive, it actually significantly improves performance of our distillation procedure (see Appendix~\ref{app:factorized-routing-ablation} for an ablation experiment). Understanding why the Burer-Monteiro reparametrization $R = R_1R_2$ makes linear routers easier to train is an interesting question for future study. We take $d_{proj} = 128$ for Pythia-70M and Gemma-270M, and we take $d_{proj} = 256$ for Pythia-410M.

\paragraph{Distillation results validate the secret sparse computation hypothesis}

In Figure~\ref{fig:main-distillation-results}, we report the fraction of variance explained by distilling the middle MLP layer of Pythia-70M and Gemma-270M to sparse MoE and dense MLP student models of varying sizes. Pythia-410M results are shown in Figure~\ref{fig:teaser-experiments}. The experiment shows that over the activation dataset $\cD^{act}$, the large language model's MLP layer can be significantly better approximated by sparse students than by dense students at a given number of active parameters. On the other hand, for the Gaussian dataset $\cD^{gauss}$ with the same mean and covariance, sparse students yield no significant gain.

\section{Discussion}\label{sec:discussion}

We discuss two high-level takeaways of our paper for deep learning research. First, distillation is a general lens for studying trained networks; see Section~\ref{sec:fruitful-paradigms}. Second, distillation can be used as a test-bed for fast and cheap experimentation with architecture design; see Section~\ref{sec:distillation-arch-design}.

\subsection{A distillation-based paradigm for deep learning theory}\label{sec:fruitful-paradigms}

A classical approach to deep learning theory proceeds by analyzing highly simplified settings -- imposing strong assumptions on the data, architecture, or optimization dynamics -- and deriving formal consequences. While this strategy has yielded valuable insights, it often abstracts away many of the properties that make modern neural networks empirically successful. In this paper, we explore an alternative path to theory development that more closely mirrors how the natural sciences build explanatory frameworks.

In fields such as physics, theories are evaluated by how well they account for the behavior of real systems, rather than by how tractable they are under idealized assumptions. Analogously, we advocate a ``model your model'' approach to deep learning theory: instead of starting from simplified constructions, we treat trained neural networks themselves as the object of study. The goal is to formulate and test hypotheses about the structure these models have actually acquired through training.

Concretely, the paradigm consists of two steps. First, one hypothesizes that a pretrained model exhibits a particular form of structure -- in this case secret sparse MoE structure. Second, the model is distilled into a restricted class of models that explicitly encode this hypothesized structure, allowing one to test whether the structure is sufficient to reproduce the original model’s behavior. In this way, distillation becomes a tool not only for compression, but for probing and validating theoretical claims about learned representations. A restricted version of this paradigm was previously advocated in \cite{boix2024towards}, which focused on the synthetic data setting where models secretly encoded small decision trees, and formulated extracting those trees as a distillation problem.

More broadly, this perspective allows for analyzing real, implemented neural networks whose behavior reflects both algorithmic and hardware considerations. As such, they may resist clean analysis in idealized mathematical settings. By grounding theory in the empirical properties of trained models themselves, the distillation-based paradigm offers a complementary route toward understanding deep learning systems as they are, rather than as simplified abstractions.

\begin{figure}[h]
    \centering

    \begin{minipage}{0.48\textwidth}
        \centering
        \textbf{Analysis of simplified setting} \\ \textbf{(classical deep learning theory)}
        \vspace{0cm}

        \begin{tikzpicture}[
            box_style/.style={
                draw, 
                text width=3.5cm, 
                align=center, 
                minimum height=2cm,
                fill=black!5 %
            },
            arrow_style/.style={
                -{Stealth[length=3mm, width=2mm]}, thick, draw
            }
        ]
            \node[box_style] (hypothesize2) {1. Make simplifying assumption on architecture, data, optimizer, etc…};
            \node[box_style, right=0.5cm of hypothesize2] (distill2) {2. Analyze simplified setting to recover or discover some phenomenon};
            \draw[arrow_style] (hypothesize2) -- (distill2);
        \end{tikzpicture}
    \end{minipage}
    \hspace{\fill} \vrule \hspace{\fill}
    \begin{minipage}{0.48\textwidth}
        \centering
        \textbf{Model your model} \\ \textbf{(this paper)}
        \vspace{0cm}

        \begin{tikzpicture}[
            box_style/.style={
                draw, 
                text width=3.5cm, 
                align=center, 
                minimum height=2cm,
                fill=blue!15 %
            },
            arrow_style/.style={
                -{Stealth[length=3mm, width=2mm]}, thick, draw
            }
        ]
            \node[box_style] (hypothesize1) {1. Hypothesize secret structure in your model};
            \node[box_style, right=0.5cm of hypothesize1] (distill1) {2. Distill your model to check for structure};
            \draw[arrow_style] (hypothesize1) -- (distill1);
        \end{tikzpicture}
    \end{minipage}

    \caption{A popular approach in deep learning theory (left) contrasted with our paper's approach (right) to understanding what a model is doing.}
    \label{fig:framework_comparison}
\end{figure}
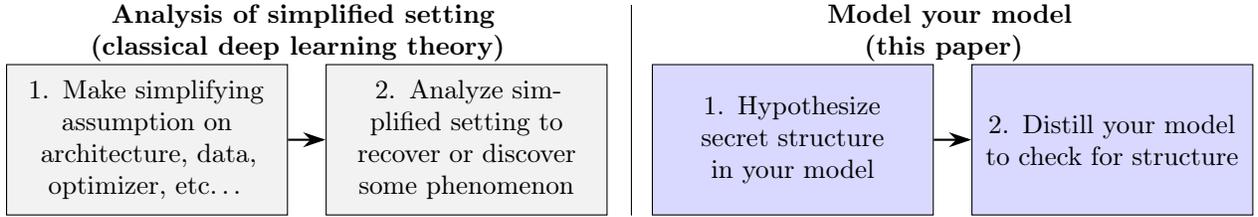

\subsection{Using distillation for fast experimentation with architecture design}\label{sec:distillation-arch-design}

Several major labs have replaced MLP layers with mixture-of-experts (MoE) layers in frontier transformer architectures to improve scalability \citep{liu2024deepseek,meta2025llama4,jiang2024mixtral,mosaic2024introducing,qwen2024qwen,snowflake2024snowflake,agarwal2025gpt}. This trend raises a natural question: why do MoE layers work so well as replacements for standard feedforward layers?

Our results suggest a simple explanation: standard MLPs already exhibit a latent MoE-like structure, making explicit MoEs a natural architectural refinement rather than a completely new component. Crucially, we show that this hypothesis can be tested in a resource-constrained academic setting. Rather than pretraining full models from scratch, our approach distills a single layer of a pretrained network, enabling targeted architectural comparisons at low computational cost. Using this distillation framework, we also find that adding a shared expert to the MoE improves performance, consistent with emerging best practices \cite{dai2024deepseekmoe}.

Taken together, these results demonstrate that distillation provides a practical and efficient tool for rapid architectural experimentation, enabling systematic comparison of candidate designs with minimal compute.

Our experimentation in the distillation setting also allows us to prescribe a new architecture to try: MoEs with \textit{low-rank routers}. When there are many experts, we find that MoEs with low-rank routers are easier to train and yield improved performance over MoEs with full-rank routers (see Section~\ref{sec:experimental-validation} and Appendix~\ref{app:factorized-routing-ablation}). The many-expert regime is increasingly important, as open-source frontier models scale the number of experts because of the performance benefits of granularity \cite{krajewski2024scaling,boix2025power}, and the computational savings of sparsity \cite{shazeer2017outrageously}. Further testing this proposed architecture is a promising direction.

\section{Further related work}\label{sec:further-related}

\paragraph{Connections to SAE design}

Recent work has shown that Sparse Autoencoders (SAEs) capture a large fraction of the variance in neural network activations. While early results focused on ReLU-based SAEs \cite{bricken2023monosemanticity}, subsequent variants—such as JumpReLU \cite{rajamanoharan2024jumping}, gated \cite{rajamanoharan2024improving}, and top-$k$ \cite{makhzani2013k,gao2024scaling,bussmann2024batchtopk} activations -- have demonstrated improved reconstruction performance at fixed sparsity levels.

This paper highlights a structural connection between SAEs and Mixture-of-Experts (MoE) models, suggesting a bidirectional opportunity: advances in SAE design may inform better MoE router architectures, and conversely, insights from MoE routing may guide the development of more expressive or efficient SAEs. Beyond standard SAEs, multi-level variants that learn nested dictionaries \cite{bussmann2025learning} may be an interesting avenue to consider, as they naturally suggest an MoE architectures with hierarchically organized experts operating at different levels of precision.

\paragraph{Mechanistic interpretability and network subcircuits}

A growing body of work in mechanistic interpretability seeks to identify functional subcircuits within neural networks -- collections of parameters or activations that are selectively engaged during computation \cite{marks2024sparse,conmy2023towards} such that the rest of the parameters can be ablated with minimal loss in performance on a task. 
In this work, we effectively find subcircuits at the level of individual MLP layers, showing that these can be well approximated by sparse models. Thus, our results show sparse structure at a finer level than mechanistic interpretability methods that identify subcircuits at the level of heads or between layers.

Another related work is the lottery ticket hypothesis \cite{frankle2018lottery}, which posits that dense neural networks contain much smaller subnetworks capable of achieving comparable performance. Their method to find those subnetworks yields an equivalent-performance model with MLP layers that have sparse weight matrices. The models from the lottery ticket hypothesis are in a different regime from the MoE models considered in this paper, since their sparsity pattern does not depend on the input, and therefore the active and total parameter counts are the same. Additionally, their sparsity pattern is not structured as in MoEs.

Another related approach to finding subcircuits, Automatic Parameter Decomposition \cite{braun2025interpretability}, explicitly finds dictionaries of parameters such that a sparse subset is active in a forward pass. In our work, we do not require the student MoE to be based on dictionary decomposition of the parameters of the initial model, so we can scale beyond toy scales to much larger student models by simply training the student MoE model.

\clearpage

\appendix

\tableofcontents

\clearpage

\section{Additional experiments}\label{app:additional-experiments}
Code for the experiments can be found in this repository: \url{https://github.com/eboix/secret_moe}. We now report results from several ablations.

\subsection{Ablation experiment showing shared expert helps}\label{app:shared-expert-ablation}

We follow the practice of the frontier open-weight architectures \cite{dai2024deepseekmoe,team2024qwen2}, which have a shared expert that is always active in their MoE architectures; see the architecture in \eqref{eq:shared-expert-moe}. Here, we validate through an ablation study that this improves performance compared to a pure mixture of experts without a shared expert \eqref{eq:moe-definition}.

\begin{table}[h]
\centering
\caption{Pythia-70m layer 3 final distillation error for different student models that have different balances of shared expert size and number of active experts in MoE. The result indicates best performance is achieved by balancing shared and MoE number of neurons. Training is for 100 epochs, and we take best of 1e-3, 3e-4, and 1e-4 learning rates.}
\begin{tabular}{c|c}
\begin{tabular}{c}\textbf{Shared expert neurons + }\textbf{Total }\\ \textbf{sparsely-activating MoE neurons }\end{tabular}& 	\textbf{Distillation fraction variance unexplained} $\downarrow$ \\
\hline
0 shared + 256 MoE &	0.0297 \\
\hline
64 shared + 192 MoE &	0.025 \\
\hline
128 shared + 128 MoE &	\textbf{0.0228} \\
\hline
192 shared + 64 MoE &	0.025
\end{tabular}
\end{table}

\subsection{Ablation experiment showing low-rank routing helps}\label{app:factorized-routing-ablation}

In Figure~\ref{fig:ablation-routing-low-rank-or-full-rank}, we compare student MoEs with full-rank routing where the routing matrix $R$ is directly trained, versus low-rank routing where we parametrize it as $R = R_1R_2$. We find that for Gaussian data there is no significant difference between the two procedures, as expected from our theory. On the other hand, for the true activation data distribution the low-rank routing generally yields much better results even though it is strictly less expressive. Since we did not  tune the rank of the routing, this indicates that our distillation fraction variance explained by MoEs may be possibly further improved.
\begin{figure}[h]

\centering

\begin{tabular}{c|c}
\textbf{Activation Data $\cD^{act}$} & \begin{tabular}{c}\textbf{Gaussian data $\cD^{gauss}$} \\ {\textbf{with 
matching mean and covariance}}\end{tabular} \\
\hline
\includegraphics[width=0.45\textwidth]{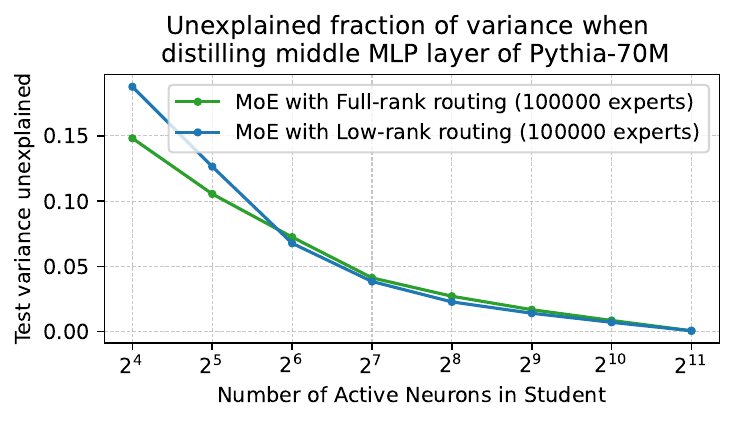} & \includegraphics[width=0.45\textwidth]{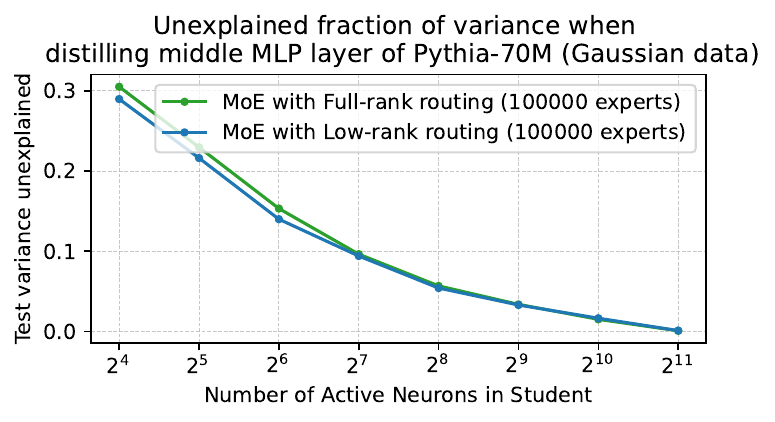} \\
\includegraphics[width=0.45\textwidth]{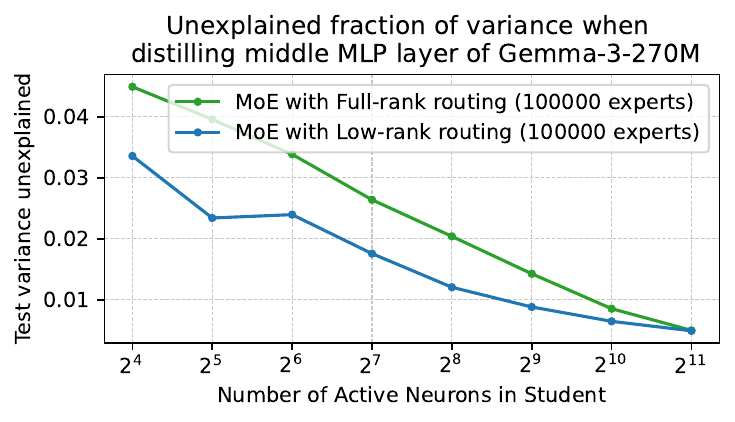} & \includegraphics[width=0.45\textwidth]{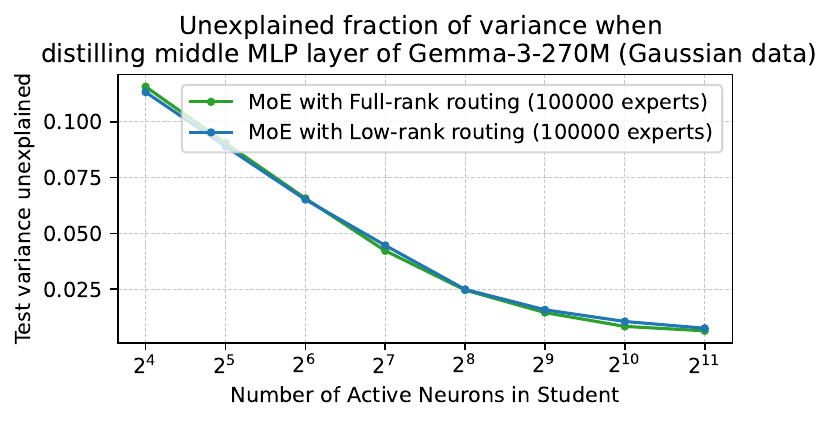} \\
\includegraphics[width=0.45\textwidth]{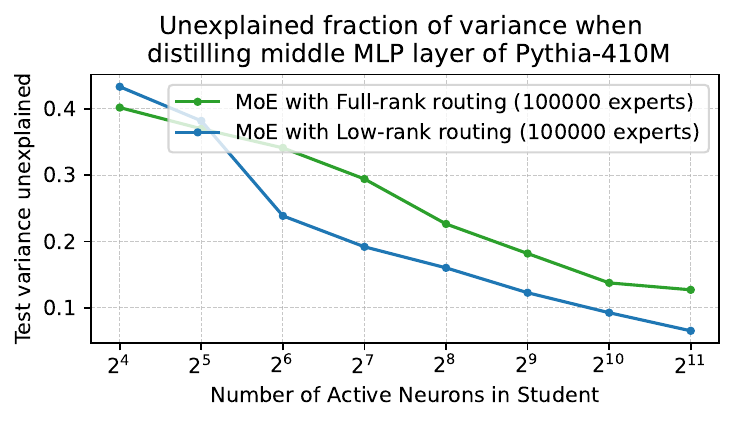} & \includegraphics[width=0.45\textwidth]{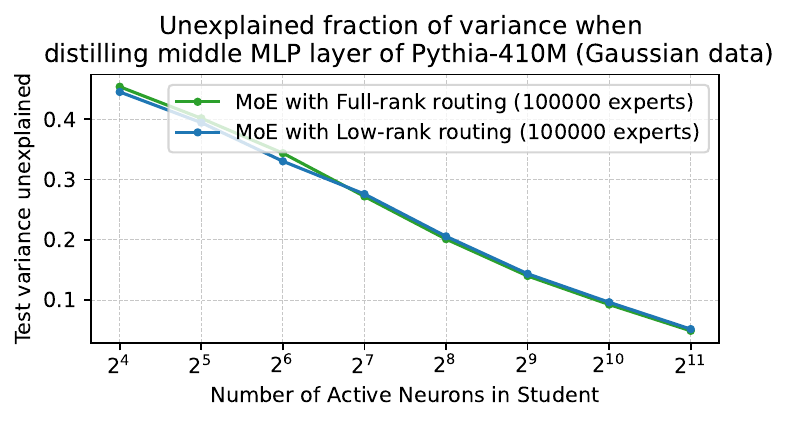}
\end{tabular}
\caption{The unexplained fraction of the variance in the outputs from distilling the middle MLP layer of Pythia-70M (first row), Gemma-270M (second row), and Pythia-410M (third row) to either an MoE with full-rank routing matrix $R \in \R^{m \times d}$, or an MoE with low-rank routing matrix $R = R_1R_2$, where $R_1$ and $R_2$ are trained. As in the main text, for low-rank routing we choose the inner dimension $128$ for Pythia-70M and Gemma-3-270M, and $256$ for Pythia-410M. Notice that distilling to a low-rank routing MoE is generally an improvement over distilling to a full-rank routing MoE, even though the former is less expressive.}\label{fig:ablation-routing-low-rank-or-full-rank}

\end{figure}
\subsection{Loss curves showing convergence of distillation procedure}

In Figures~\ref{fig:loss-curve-stabilize1} through \ref{fig:loss-curve-stabilize6}, we plot the test loss curves during training for several distillations of models to student MLP and MoE models. These plots show that generally the distillation to MLP models has a loss that stabilizes quickly and is fairly independent of the learning rate. On the other hand, the performance of the distillation to MoE models has a higher variability with the learning rate, but still converges. Part of the effect of the convergence is due to cosine learning rate decay, but the loss curves generally seem to stabilize early on -- especially for the MLP student distillations.

\begin{figure}[h]

\centering

\begin{tabular}{c|c}
\textbf{Activation Data $\cD^{act}$} & \begin{tabular}{c}\textbf{Gaussian data $\cD^{gauss}$} \\ {\textbf{with 
matching mean and covariance}}\end{tabular} \\
\hline
\includegraphics[width=0.45\textwidth]{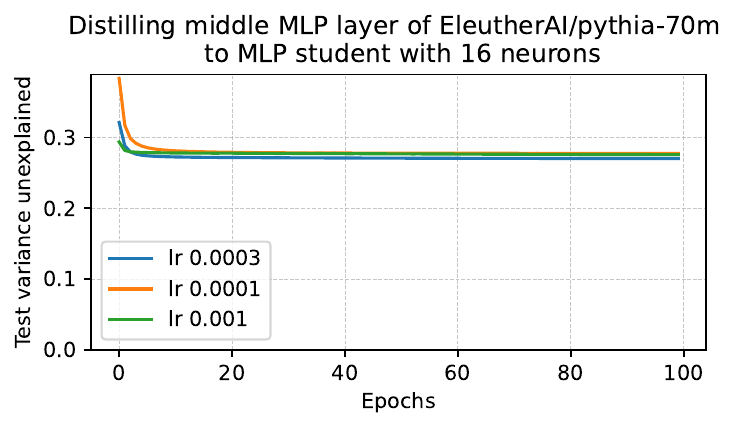} & \includegraphics[width=0.45\textwidth]{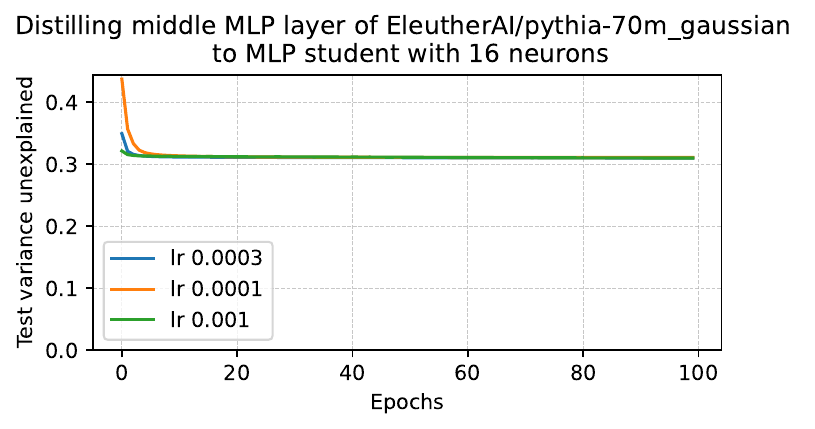} \\
\includegraphics[width=0.45\textwidth]{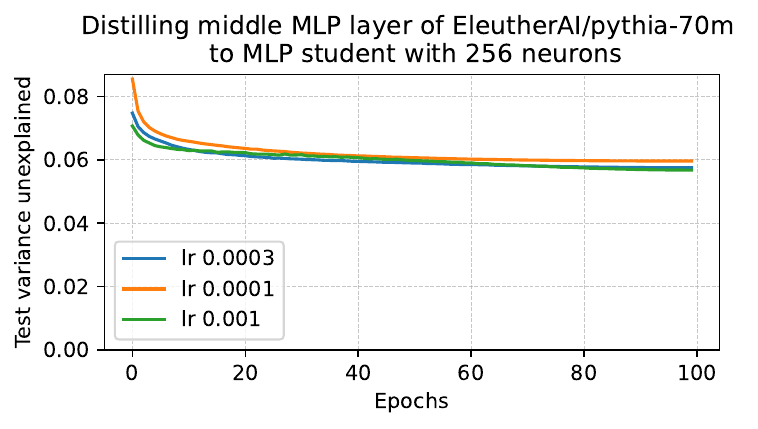} & \includegraphics[width=0.45\textwidth]{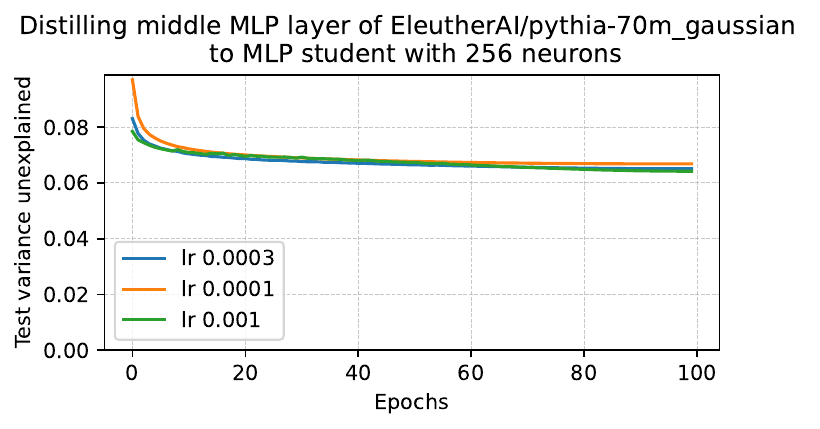}
\end{tabular}
\caption{Test variance unexplained by iteration when training MLP students on Pythia-70M with Adam for 100 epochs. The loss curves converge (although this is due in part to cosine learning rate decay).}\label{fig:loss-curve-stabilize1}
\end{figure}

\begin{figure}[h]

\centering

\begin{tabular}{c|c}
\textbf{Activation Data $\cD^{act}$} & \begin{tabular}{c}\textbf{Gaussian data $\cD^{gauss}$} \\ {\textbf{with 
matching mean and covariance}}\end{tabular} \\
\hline
\includegraphics[width=0.45\textwidth]{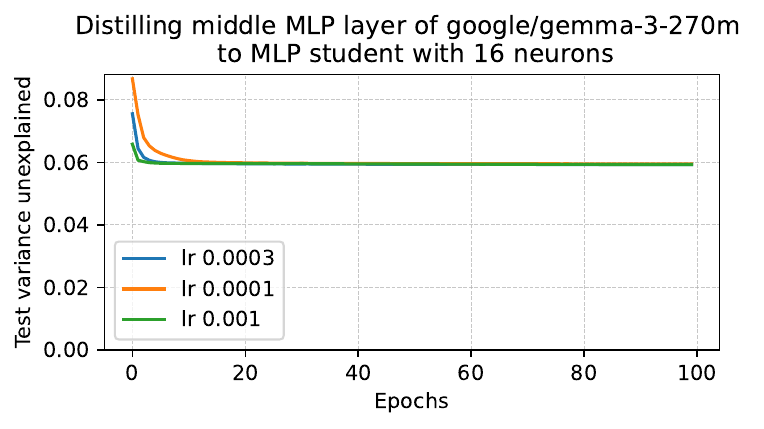} & \includegraphics[width=0.45\textwidth]{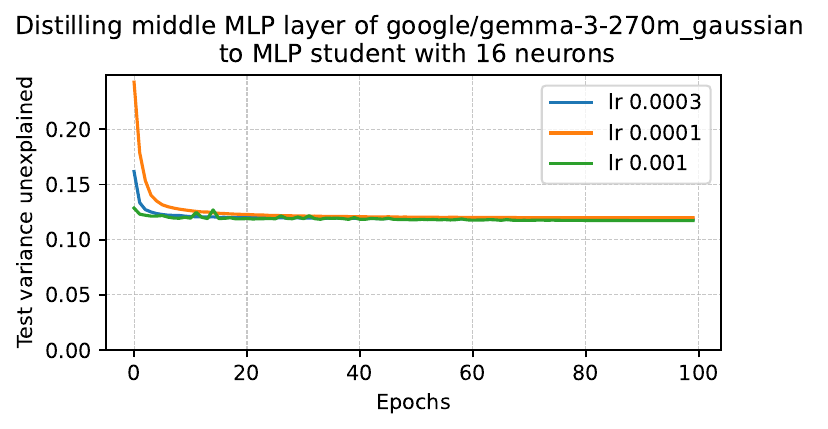} \\
\includegraphics[width=0.45\textwidth]{figs/distillation_mlp_learning_curve_google_gemma-3-270m_gaussian_layer9_hd16.pdf} & \includegraphics[width=0.45\textwidth]{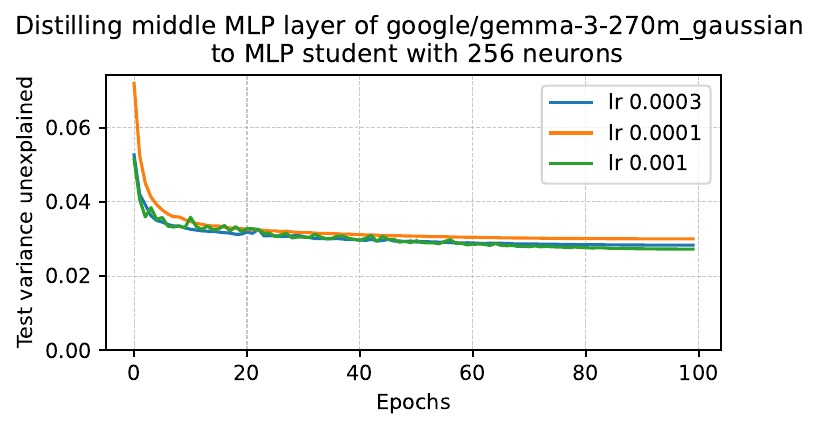}
\end{tabular}
\caption{Test variance unexplained by iteration when training MLP students on Gemma-270M with Adam for 100 epochs. The loss curves converge (although this is due in part to cosine learning rate decay).}\label{fig:loss-curve-stabilize2}
\end{figure}

\begin{figure}[h]

\centering

\begin{tabular}{c|c}
\textbf{Activation Data $\cD^{act}$} & \begin{tabular}{c}\textbf{Gaussian data $\cD^{gauss}$} \\ {\textbf{with 
matching mean and covariance}}\end{tabular} \\
\hline
\includegraphics[width=0.45\textwidth]{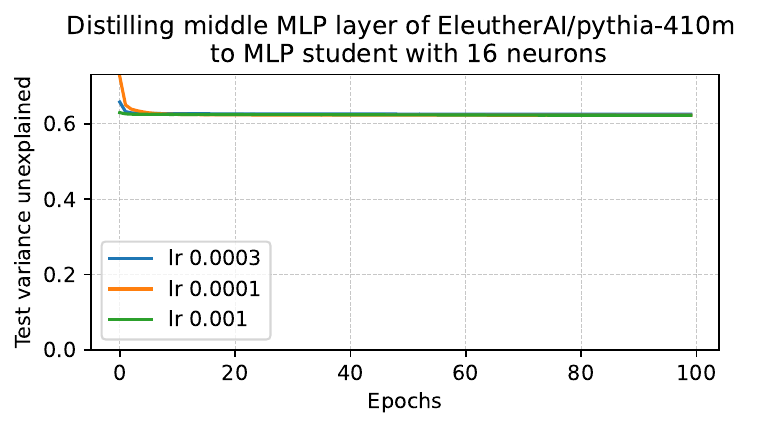} & \includegraphics[width=0.45\textwidth]{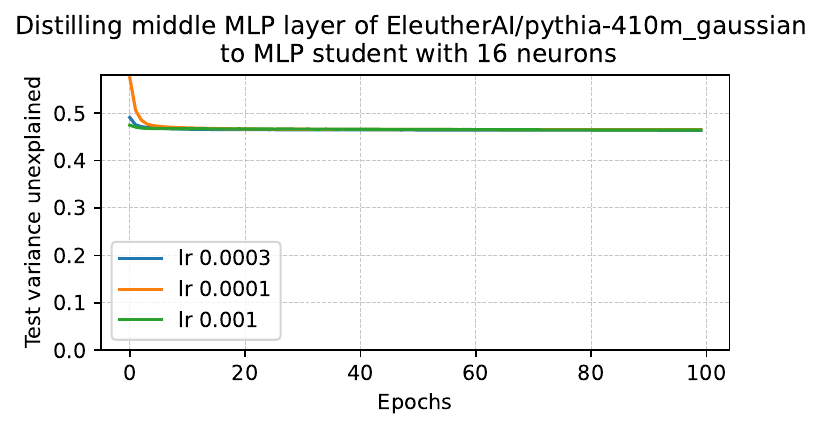} \\
\includegraphics[width=0.45\textwidth]{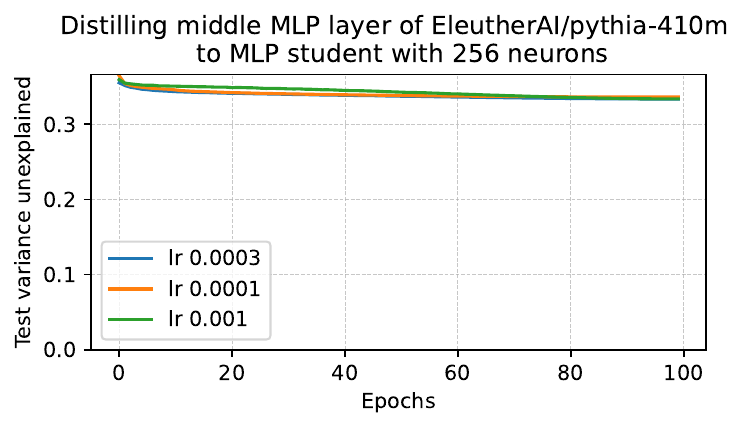} & \includegraphics[width=0.45\textwidth]{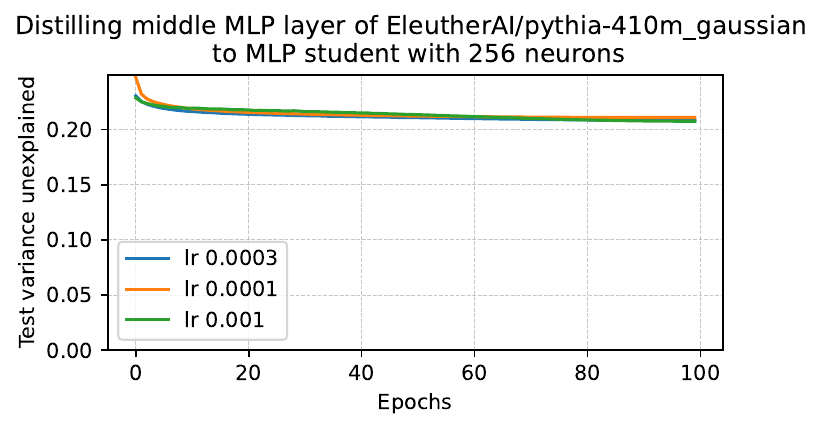}
\end{tabular}
\caption{Test variance unexplained by iteration when training MLP students on Pythia-410M with Adam for 100 epochs. The loss curves converge (although this is due in part to cosine learning rate decay).}\label{fig:loss-curve-stabilize3}
\end{figure}

\begin{figure}[h]

\centering

\begin{tabular}{c|c}
\textbf{Activation Data $\cD^{act}$} & \begin{tabular}{c}\textbf{Gaussian data $\cD^{gauss}$} \\ {\textbf{with 
matching mean and covariance}}\end{tabular} \\
\hline
\includegraphics[width=0.45\textwidth]{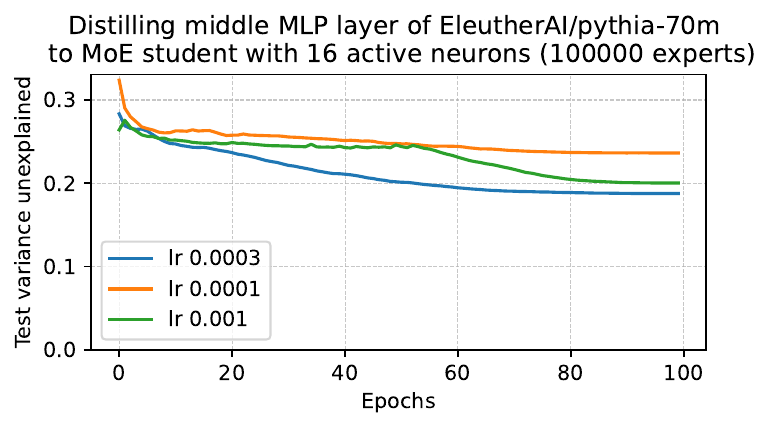} & \includegraphics[width=0.45\textwidth]{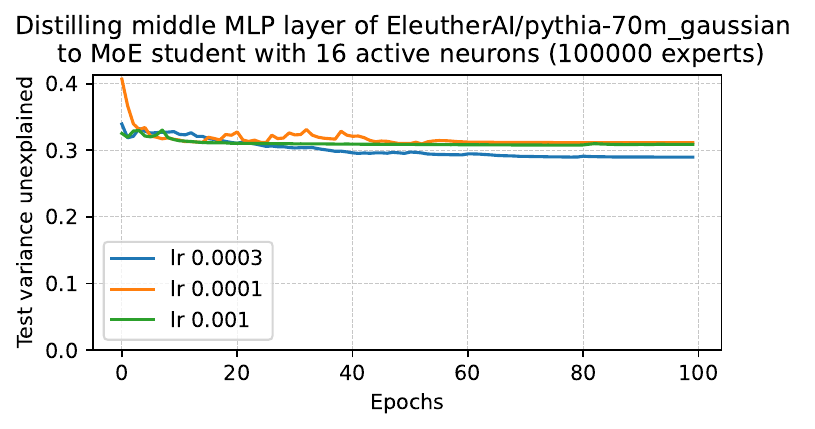} \\
\includegraphics[width=0.45\textwidth]{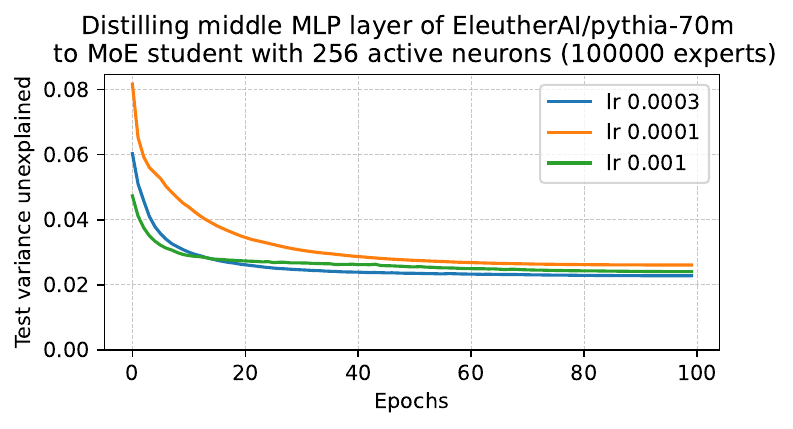} & \includegraphics[width=0.45\textwidth]{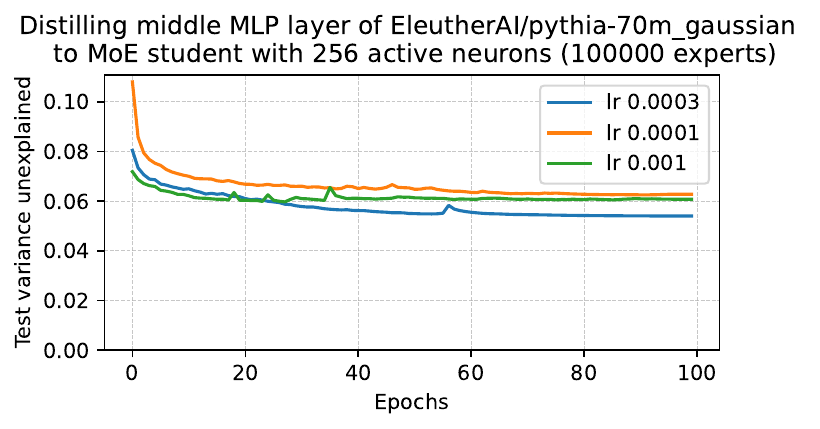}
\end{tabular}
\caption{Test variance unexplained by iteration when training MoE students on Pythia-70M with Adam for 100 epochs. The loss curves converge (although this is due in part to cosine learning rate decay).}\label{fig:loss-curve-stabilize4}
\end{figure}

\begin{figure}[h]

\centering

\begin{tabular}{c|c}
\textbf{Activation Data $\cD^{act}$} & \begin{tabular}{c}\textbf{Gaussian data $\cD^{gauss}$} \\ {\textbf{with 
matching mean and covariance}}\end{tabular} \\
\hline
\includegraphics[width=0.45\textwidth]{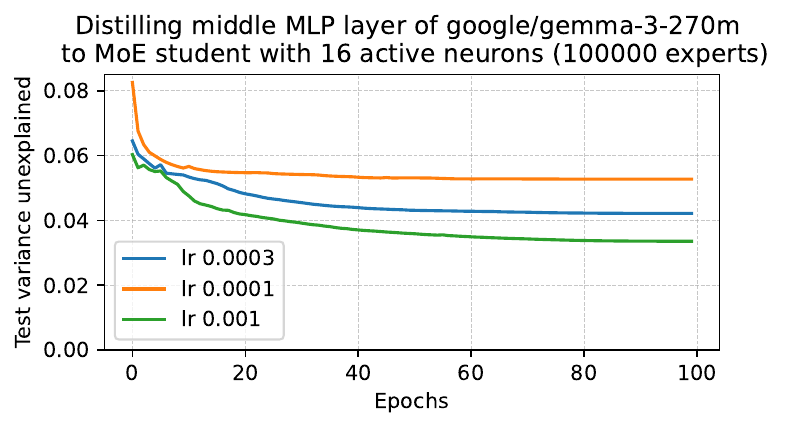} & \includegraphics[width=0.45\textwidth]{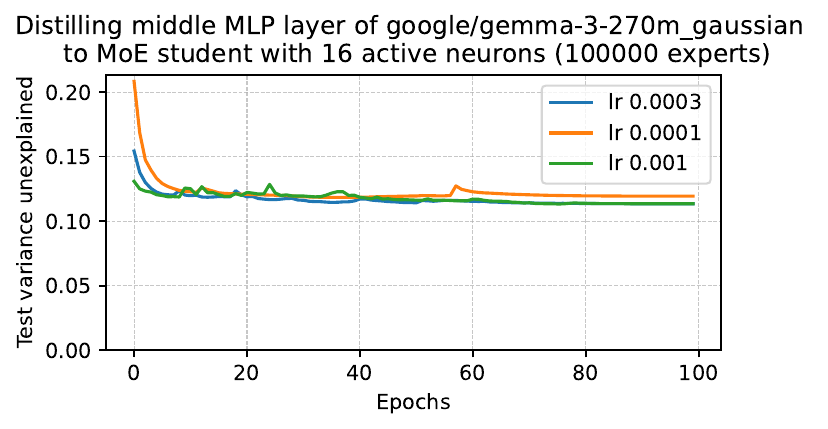} \\
\includegraphics[width=0.45\textwidth]{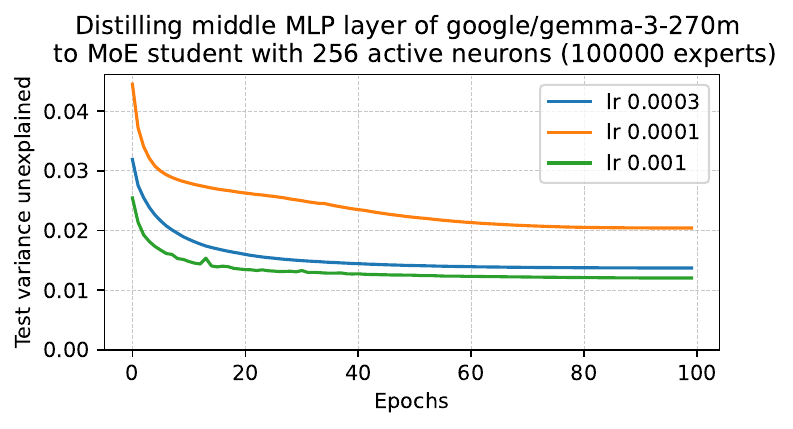} & \includegraphics[width=0.45\textwidth]{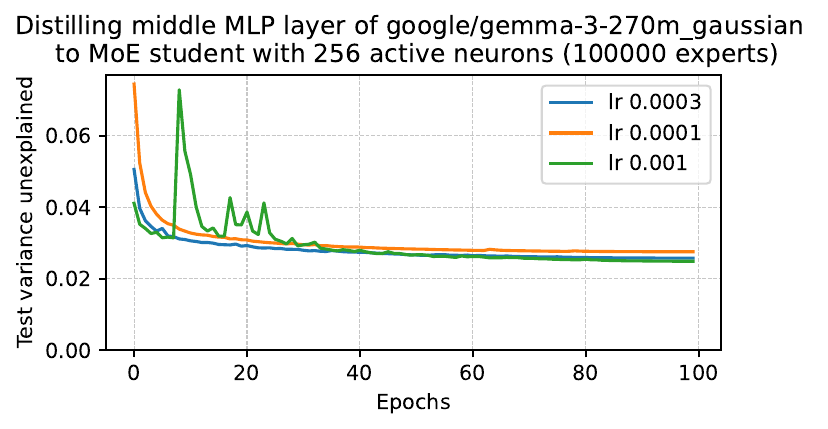}
\end{tabular}
\caption{Test variance unexplained by iteration when training MoE students on Gemma-270M with Adam for 100 epochs. The loss curves converge (although this is due in part to cosine learning rate decay).}\label{fig:loss-curve-stabilize5}
\end{figure}

\begin{figure}[h]

\centering

\begin{tabular}{c|c}
\textbf{Activation Data $\cD^{act}$} & \begin{tabular}{c}\textbf{Gaussian data $\cD^{gauss}$} \\ {\textbf{with 
matching mean and covariance}}\end{tabular} \\
\hline
\includegraphics[width=0.45\textwidth]{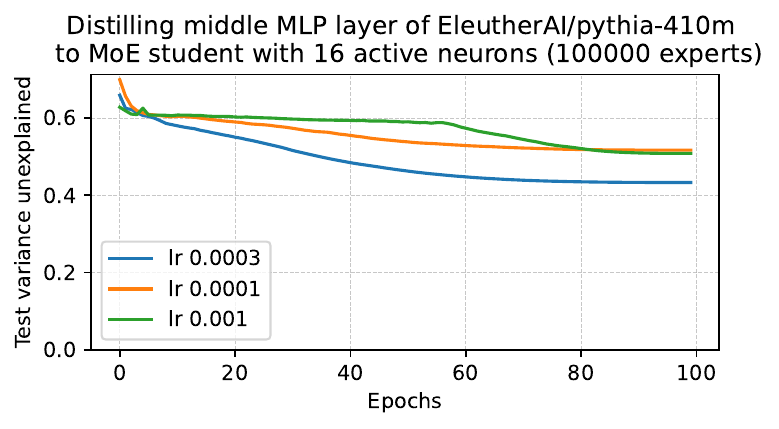} & \includegraphics[width=0.45\textwidth]{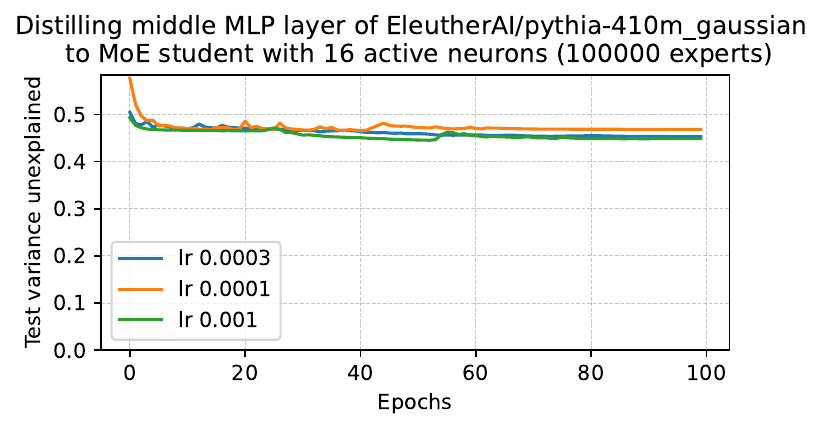} \\
\includegraphics[width=0.45\textwidth]{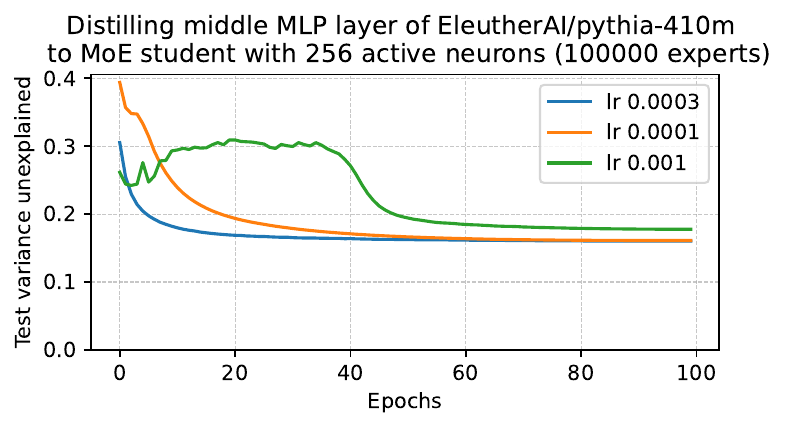} & \includegraphics[width=0.45\textwidth]{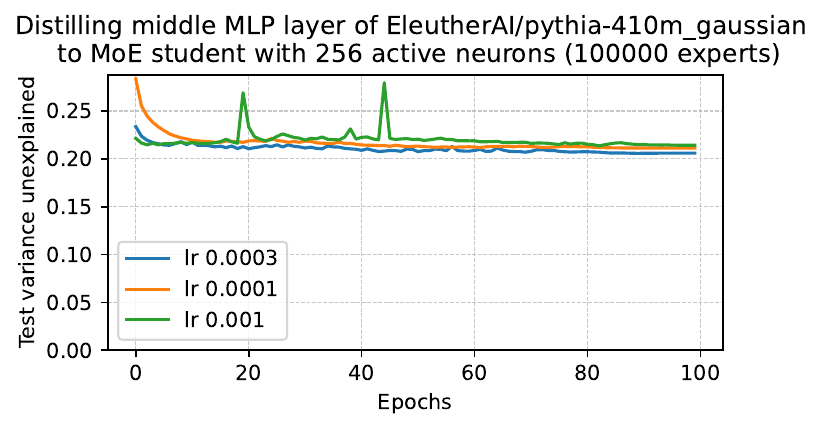}
\end{tabular}
\caption{Test variance unexplained by iteration when training MoE students on Pythia-410M with Adam for 100 epochs. The loss curves converge (although this is due in part to cosine learning rate decay).}\label{fig:loss-curve-stabilize6}
\end{figure}

\clearpage
\section{Proof of Theorem~\ref{thm:inapproximability-moe-gaussian}}\label{app:gaussian-distribution-identity-inapproximable}

In this section, we prove Theorem~\ref{thm:inapproximability-moe-gaussian}, which is restated below for convenience.

\begin{theorem}[Inapproximability of identity by sparse MoEs under Gaussian data distribution]\label{thm:inapproximability-moe-gaussian-restated}
There are universal constants $c,c' > 0$ such that the following is true. Under isotropic Gaussian input distribution $D = N(0,I_d/d)$, \textbf{there is no} $(m,k)$-MoE with hard gating function and a number of active neurons $kd_{exp} < d/2$ and number of expert configurations $m^k \leq \exp(cd)$, that can $c'$-approximate the function $f^*(x) = x$.
\end{theorem}

\subsection{Technical results from \cite{boix2025power}}
The definitions and technical claim that we will use from \cite{boix2025power} are as follows.
\begin{definition}
For a distribution $\mu$, and a measurable set $U$, let $\mu|_U$ be the probability measure from restricting $\mu$ to $U$. I.e., $\mu|_U(A) = \mu(A \cap U) / \mu(U)$ for all measurable sets $A$. Given a distribution $\mu$ and measurable set of nonzero measure, additionally define $\Sigma_U = \cov(X,X)$ for $X \sim \mu|_U$ to be the conditional covariance.
\end{definition}

The technical claim that we use is an error bound for approximating linear functions by a nonlinear function that depends on a propert subspace of the input.

\begin{proposition}[Lemma C.2 of \cite{boix2025power}]\label{prop:inapprox-prop-from-boix-rigollet}
There are universal constants $C,c > 0$ such that the following is true for $\mu = N(0,I_d / d)$. Let $U \subseteq \R^d$ be a measurable set, and let $\Pi \in \R^{d \times d}$ be a projection matrix to a subspace of dimension $p \leq 99d / 100$, and let $h : \R^d \to \R^d$ be a measurable function, and let $A \in \R^{d \times d}$ be a linear transformation. Then,
\begin{align*}
\E_{x \sim \mu|_U} \|Ax - h(\Pi x)\|^2 \geq \frac{c}{d} \sum_{i \geq C (1 + \log(1/\mu(U)))} \sigma_i^2(A\Pi^{\perp})\,,
\end{align*}
where $\Pi^{\perp} \in \R^{d \times d}$ is the orthogonal projection to $\Pi$.
\end{proposition}

\subsection{Proof of Theorem~\ref{thm:inapproximability-moe-gaussian}}

We may now proceed to prove the theorem. Let $\fmoe$ be the $(m,k)$-MoE that approximates the identity function on distribution $\mu = N(0,I_d / d)$. For any $S \subseteq [m]$, let $U_S = \{x \in \R^d : \mathrm{support}(g(x)) = S\}$ be the region on which the subset of experts $S$ is active.

Define the set of experts that have lower-bounded probability to be 
\begin{align*}
\mathcal{S} = \{S : \P[x \in U_S] \geq 1/(3(m+1)^k)\}\,.
\end{align*}

For any $S \subseteq [m]$ with $|S| \leq k$, let $f_S : \R^{kd_{exp}} \to \R^d$ be a function and let $\Pi : \R^{kd_{exp} \times d}$ be a projection such that 
\begin{align*}
\fmoe(x) = f_S(\Pi_S x) \mbox{ for all } x \in U_S\,.
\end{align*}
This is guaranteed, since $\fmoe$ is a $(m,k)$-MoE with expert dimension $d_{exp}$.

By Proposition~\ref{prop:inapprox-prop-from-boix-rigollet}, there are universal constants $c,C > 0$ such that, for any $S \in \mathcal{S}$ we have
\begin{align*}
\E_{x \sim \mu|_{U_S}}[\|f^*(x) - \fmoe(x)\|^2] &= \E_{x \sim \mu|_{U_S}}[\|x - f_S(\Pi_S x)\|^2] \\
&\geq \frac{c}{d} \sum_{i \geq C(1 + \log(1/\mu(U_S)))} \sigma_i^2(\Pi_S^{\perp}) \\
&\geq \frac{c}{d} (d- C(1 + \log(1/\mu(U_S))) - \rank(\Pi_S)) \\
&\geq \frac{c}{d} (d- C(1 + \log(1/\mu(U_S))) - kd_{exp}) \\
&\geq \frac{c}{d} (d- C(1 + k \log(3(m+1))) - kd_{exp})\,.
\end{align*}
Therefore,
\begin{align*}
\E_{x \sim \mu}[\|f^*(x) - \fmoe(x)\|^2] &= \sum_{S \subseteq [m]} \P[x \in U_S] \E_{x \sim \mu|_{U_S}}[\|f^*(x) - \fmoe(x)\|^2] \\
&\geq \sum_{S \in \cS} \P[x \in U_S] \E_{x \sim \mu|_{U_S}}[\|f^*(x) - \fmoe(x)\|^2] \\
&\geq \sum_{S \in \cS} \P[x \in U_S] \frac{c}{d} (d- C(1 + k \log(3(m+1))) - kd_{exp}) \\
&\geq \frac{c}{d} (d- C(1 + k \log(3(m+1))) - kd_{exp}) (1 - \sum_{S \not\in \cS} \P[x \in U_S]) \\
&\geq \frac{c}{d} (d- C(1 + k \log(3(m+1))) - kd_{exp}) (1 - (m+1)^k (1/(3(m+1)^k))) \\
&\geq \frac{2c}{3d} (d- C(1 + k \log(3(m+1))) - kd_{exp}) \\
&\geq \frac{c'}{d} (d-C'k\log(m)-kd_{exp})\,,
\end{align*}
for universal constants $C,c' > 0$.

Theorem~\ref{thm:inapproximability-moe-gaussian} follows from this inequality. \qed

\section{Extension of MoE approximability result to nonlinear functions}\label{app:moe-nonlinear-extension}

We now strengthen the theoretical evidence for Hypothesis~\ref{hyp:main-hypothesis} by showing that the approximability result holds beyond linear functions, to a class of nonlinear functions that have the property that interactions between the dictionary features are of low complexity. Our result here naturally generalizes Theorem~\ref{thm:approx-linear-with-sparse-dictionary-data} to nonlinear functions.

\subsection{Basic definitions}
For simplicity, we consider target functions that are homogeneous degree-$p$ polynomials (this result can be generalized to non-homogenous polynomials but with more notational heaviness). \begin{definition} Homogeneous degree-$p$ polynomials are functions $f : \R^d \to \R^d$ of the form
\begin{align}\label{eq:homogeneous-poly}
[f(x)]_{j_1} = \sum_{j_2,\ldots,j_{p+1} \in [d]} A_{j_1,\ldots,j_{p+1}} x_{j_1} x_{j_2} \cdots x_{j_{p+1}}\,,
\end{align}
for some tensor $A \in (\R^{d})^{\otimes (p+1)}$.
\end{definition}

When $p = 1$, notice that we recover linear functions $f(x) = Ax$ for a matrix $A \in \R^{d \times d}$. We define an operator norm, a tensor-vector product, and a rank for these tensors as follows:
\begin{definition}
For a tensor $A \in (\R^d)^{\otimes (p+1)}$, we define the operator norm $$\|A\|_{\mathrm{op}} = \max_{v_1,\ldots,v_{p+1} \in \R^d \sm 0} \frac{\sum_{j_1,\ldots,j_{p+1} \in [d]} A_{j_1,\ldots,j_{p+1}} [v_1]_{j_1} [v_2]_{j_2} \cdots [v_{p+1}]_{j_{p+1}}}{\|v_1\|\|v_2\| \cdots \|v_{p+1}\|},$$
\end{definition}
\begin{definition}[Tensor-vector product]\label{def:tensor-vector-multiplication}
Given a tensor $A \in (\R^d)^{\otimes (p+1)}$ and a vector $v \in \R^d$, we define the tensor-vector product along the last index $Av \in (\R^{d})^{\otimes p}$, which is given by
\begin{align*}
[Av]_{j_1,\ldots,j_p} = \sum_{j_{p+1} \in [d]} A_{j_1,j_2,\ldots,j_p,j_{p+1}} v_{j_{p+1}}
\end{align*}
\end{definition}
\begin{definition}[Tensor rank]
We say that $B \in (\R^{d})^{\otimes p}$, is rank at most $r$ if there are vectors $u_{ij} \in \R^d$ for $i \in [r]$ and $j \in [p]$ such that
\begin{align*}
B = \sum_{i=1}^r u_{i1} \otimes u_{i2} \otimes \cdots \otimes u_{ip}\,.
\end{align*}
\end{definition}

With these definitions in hand, we may define what it means for a tensor $A$ to have ``low rank'' interactions between the features in the dictionary.

\begin{definition}[Low-rank feature interactions]
Given a dictionary $v_1,\ldots,v_m \in \R^d$ and a homogeneous polynomial $f(x)$ of the form \eqref{eq:homogeneous-poly} for a tensor $A \in (\R^{d})^{\otimes (p+1)}$, we say that the interactions between dictionary features are rank $\leq r$ if
\begin{align*}
\mbox{rank}(Av_i) \leq r\,,
\end{align*}
for all $i \in [m]$.
\end{definition}

Notice that our notions of operator norm, tensor-vector product and rank agree with the corresponding standard definitions of operator norm, matrix-vector product and rank for matrices when $p = 1$. Thus, our results in this section will naturally generalize the case of linear functions developed in Theorem~\ref{thm:approx-linear-with-sparse-dictionary-data}. Indeed, one can observe that linear functions are a special case of functions with rank-1 feature interactions.

\subsection{Statement of approximability result for nonlinear functions}
Our generalization of Theorem~\ref{thm:approx-linear-with-sparse-dictionary-data} to nonlinear functions is as follows.

\begin{theorem}[Low-rank feature interaction functions are approximable by sparse MoEs]\label{thm:approx-nonlinear-with-sparse-dictionary-data}
Suppose that the distribution $D$ is supported in the unit ball, and is an $(m,k)$-dictionary-sparse distribution for a $(\gamma,k)$-approximately-orthogonal dictionary.
Suppose also that $f^* : \R^d \to \R^d$ is a homogeneous polynomial given by a tensor $A \in (\R^d)^{\otimes (p+1)}$ according to \eqref{eq:homogeneous-poly} and has rank-$r$ interactions between dictionary features.

Then, there is a hard-gated $(m,k)$-MoE with $d_{expert} \leq O_p(r)$ that $\gamma \|A\|_{\mathrm{op}}$-approximates $f^*$.
\end{theorem}
\begin{proof}[Proof outline]
The proof proceeds by having the $m$ experts of the MoE correspond to the features in the dictionary. When a certain feature is active, the corresponding expert is loaded and captures the interaction of that feature with the rest of the input. The interaction rank condition ensures that this interaction can be represented with a width-$O_p(r)$ expert that has activation function $\sigma(t) = t^p$. The full proof is in Appendix~\ref{app:nonlinear-representation-proof}.
\end{proof}

\begin{remark}[Recovering Theorem~\ref{thm:approx-linear-with-sparse-dictionary-data}]
In the case of linear functions $f^*(x) = Ax$ where $A \in \R^{d \times d}$, the function has rank-1 interactions. This is because $Av_i$ is a vector for all $i \in [m]$, and so it is trivially a rank-1 tensor. Thus, the above theorem recovers our previous result in Theorem~\ref{thm:approx-linear-with-sparse-dictionary-data}, up to a constant factor blow-up in the expert width.
\end{remark}

\begin{remark}[Comparing active parameters of MLP versus MoE for generic homogeneous polynomials]
Generically, we should expect a tensor $A \in (\R^{d})^{\otimes (p+1)}$ to have rank $\Theta(d^{p+1})$ and the interactions to have rank $\Theta(d^p)$. Therefore, an MLP should require width $\Omega(d^{p+1})$ to express the polynomial, and an MoE should on the other hand require active number of neurons $O(kd^p)$. This implies a factor of $\Theta(d)$ decrease in the number of active parameters by using an MoE when we have constant dictionary sparsity $k = O(1)$.
\end{remark}

\subsection{Proof of Theorem~\ref{thm:approx-nonlinear-with-sparse-dictionary-data}}\label{app:nonlinear-representation-proof}

\subsubsection{Preliminary definitions and lemmas on tensor rank}

First, let us make a technical definition which is a variation on tensor rank, and let us prove a key lemma relating it to tensor rank.

\begin{definition}[Last-$(p-1)$-symmetric tensor rank]
Given $B \in (\R^{d})^{\otimes p}$, we say that has last-$(p-1)$-rank at most $r$ if there are vectors $w_i, u_{i} \in \R^d$ for $i \in [r]$ and $j \in [p]$ such that
\begin{align*}
B = \sum_{i=1}^r w_{i} \otimes u_{i}^{\otimes (p-1)}\,.
\end{align*}
\end{definition}
\begin{lemma}\label{lem:rank-to-symm-rank}
For any $p$, there is a constant $C_p > 0$ such that the following is true. Suppose that $B \in (\R^{d})^{\otimes p}$ is rank-$r$ and is symmetric in its last $(p-1)$ indices. Then its last-$(p-1)$-symmetric tensor rank is at most $C_p r$.
\end{lemma}
\begin{proof}
For any $v,u_1,\ldots,u_{p-1} \in \R^d$, define $B = v \otimes \E_{\sigma \in S_{p-1}} [u_{\sigma(1)}\otimes \cdots \otimes u_{\sigma(p-1)}]$.
It suffices to show that the last-$(p-1)$-symmetric tensor rank of $B$ is bounded by a constant $C_p$ depending only on $p$.

For any $t_1,\ldots,t_{p-1} \in \R$, define
\begin{align*}
B^{(t_1,\ldots,t_{p-1})} = v \otimes (\sum_{i=1}^{p-1} t_i u_{i})^{\otimes (p-1)}\,.
\end{align*}
Notice that this is a total-degree-$(p-1)$ polynomial in $t_1,\ldots,t_{p-1}$, over the ring of tensors $(\R^{d})^{\otimes p}$. By polynomial interpolation, there is a linear combination of $\leq C_p$ of these tensors that yields the $t_1t_2 \cdots t_{p-1}$ term of $B^{(t_1,\ldots,t_{p-1})}$, which is the tensor $B$.

More explicitly, given a polynomial $h(t_1,\ldots,t_{p-1})$ of total degree at most $(p-1)$ over some ring $R$, we can extract the $t_1\cdots t_{p-1}$ term by viewing the polynomial as a polynomial in $t_{p-1}$ over the ring $R[t_1,\ldots,t_{p-2}]$, extracting the $t_{p-1}$ term which is an element of $R[t_1,\ldots,t_{p-2}]$, and iterating. This $t_{p-1}$ term can be computed as
\begin{align*}
\sum_{i=1}^p c_i h(t_1,\ldots,t_{p-2},i)\,,
\end{align*}
where $c_i$ are coefficients such that for any $0 \leq d \leq p-1$ we have $\sum_{i=1}^p c_i i^j = \delta_{j,1}$. Such coefficients are guaranteed by the invertibility of the Vandermonde matrix.
\end{proof}

\subsubsection{Equivalence between low tensor rank and existence of small MLP}

\begin{lemma}[Small tensor rank means there is a small network]\label{lem:low-rank-to-mlp}
For any $p$, there is a constant $C_p > 0$ such that the following is true. Suppose that $B \in (\R^d)^{\otimes (p+1)}$ has tensor rank $r$. Then there is an MLP with activation $\sigma(t) = t^p$ and width $\leq C_p r$ computing $f : \R^d \to \R^d$ given by
\begin{align*}
[f(x)]_{j_1} = \sum_{j_2,\ldots,j_{p+1}} B_{j_1,\ldots,j_{p+1}} x_{j_2} \dots x_{j_{p+1}}\,.
\end{align*}
\end{lemma}
\begin{proof}

First, let us define the version of $B$ that is symmetrized in the last $(p-1)$ entries by $\tilde{B} \in (\R^d)^{\otimes (p+1)}$ with
\begin{align*}
\tilde{B}_{j_1,\ldots,j_{p+1}} = \frac{1}{p!} \sum_{\tau \in S_p} B_{j_1,j_{\sigma(1)+1},\ldots,j_{\sigma(p)+1}}\,,
\end{align*}
where the sum is over permutations over $[p]$. Notice that 
\begin{align*}
\rank(\tilde{B}) \leq p! \rank(B) \leq p! r\,,
\end{align*}
and that we may equivalently write the target function as 
\begin{align*}
[f(x)]_{j_1} = \sum_{j_2,\ldots,j_{p+1}} \tilde{B}_{j_1,\ldots,j_{p+1}} x_{j_2} \dots x_{j_{p+1}}\,.
\end{align*}

By Lemma~\ref{lem:rank-to-symm-rank}, there is a constant $C_p > 0$ such that $\tilde{B}$ has last-$(p-1)$-symmetric rank $r' \leq C_pr$. Let $w_i,u_i \in \R^d$ for $i \in [r']$ be such that
\begin{align*}
\tilde{B} = \sum_{i=1}^{r'} w_i \otimes u_i^{\otimes (p-1)}\,.
\end{align*}
It follows that 
\begin{align*}
[f(x)]_{j_1} &= \sum_{i=1}^{r'}\sum_{j_2,\ldots,j_{p+1}} [w_i]_{j_1} [u_i]_{j_2} \dots [u_i]_{j_{p+1}} x_{j_2} \dots x_{j_{p+1}} \\
&= \sum_{i=1}^{r'}[w_i]_{j_1} (u_i \cdot x)^{p}\,,
\end{align*}
which is an MLP of width $\leq C_p r$, with activation function $\sigma(t) = t^p$.
\end{proof}

\subsubsection{Proof of Theorem~\ref{thm:approx-nonlinear-with-sparse-dictionary-data}}

We are now ready to prove Theorem~\ref{thm:approx-nonlinear-with-sparse-dictionary-data}.

\begin{proof}[Proof of Theorem~\ref{thm:approx-nonlinear-with-sparse-dictionary-data}]

Let the MoE have $m$ experts, each corresponding to one of the elements of the dictionary. Let the $i$th expert compute $f_i : \R^d \to \R^d$, given by
$$[f_i(x)]_{j_1} = \sum_{j_2,\ldots,j_{p+1} \in [d]} A_{j_1,\ldots,j_{p+1}} x_{j_2} x_{j_3} \cdots x_{j_{p}} [\hat{v}_i]_{j_{p+1}} (\hat{v}_i \cdot x)\,.$$

Next, let the hard gating function be such that, on input $x$, the gating function activates $k$ distinct experts $i_1,\ldots,i_k$, such that $x \in \mathrm{span}\{v_{i_1},\ldots,v_{i_k}\}$. By definition, $\fmoe(x) = \sum_{l=1}^k f_{i_l}(x)$. This implies
\begin{align*}
[\fmoe(x) - f^*(x)]_{j_1} &= \sum_{j_2,\ldots,j_{p+1} \in [d]} A_{j_1,\ldots,j_{p+1}} x_{j_2}x_{j_3} \cdots x_{j_{p}} ((\sum_{l=1}^k [\hat{v}_{i_l}]_{j_{p+1}} (\hat{v}_{i_l} \cdot x)) - x_{j_{p+1}})\,.
\end{align*}
If we define the error vector $\xi(x) = \sum_{l=1}^k \hat{v}_{i_l} \hat{v}_{i_l}^{\top} x - x$, which we recall is guaranteed by $(\gamma,k)$-approximate-orthogonality to be of norm $\|\xi(x)\| \leq \gamma \|x\| \leq \gamma$, then
\begin{align*}
[\fmoe(x) - f^*(x)]_{j_1} = \sum_{j_2,\ldots,j_{p+1} \in [d]} A_{j_1,\ldots,j_{p+1}} x_{j_2}x_{j_3} \cdots x_{j_{p}} [\xi(x)]_{j_{p+1}}\,.
\end{align*}
So, by the definition of the tensor operator norm,
\begin{align*}
\|\fmoe - f^*(x)\| &\leq \|A\|_{\mathrm{op}} \|x\|^{p-1} \|\xi(x)\| \leq \gamma \|A\|_{\mathrm{op}}\,.
\end{align*}

It remains to show that each expert $f_i$ can be implemented by an MLP of width at most $C_p r$. To see this, note that if we define $B^{(i)} = A \hat{v}_i \in (\R^{d})^{\otimes p}$ (with tensor-vector multiplication as in Definition~\ref{def:tensor-vector-multiplication}), then
\begin{align*}
[f_i(x)]_{j_1} = \sum_{j_2,\ldots,j_{p+1}} B^{(i)}_{j_1,\ldots,j_{p}} [\hat{v}_{i}]_{j_{p+1}} x_{j_2} \dots x_{j_{p+1}}\,.
\end{align*}
In other words, if we define $C^{(i)} \in (\R^{d})^{\otimes (p+1)}$ by $C^{(i)}_{j_1,\ldots,j_{p+1}} = B^{(i)}_{j_1,\ldots,j_p} [\hat{v}_i]_{j_{p+1}}$, then
\begin{align*}
[f_i(x)]_{j_1} = \sum_{j_2,\ldots,j_{p+1}} C^{(i)}_{j_1,\ldots,j_{p+1}} x_{j_2} \dots x_{j_{p+1}}\,.
\end{align*}

Notice that 
\begin{align*}
\rank(C^{(i)}) \leq \rank(B^{(i)}) \leq r\,,
\end{align*}
by the assumption of rank-$r$ interactions between features. Applying Lemma~\ref{lem:low-rank-to-mlp} to $B^{(i)}$ implies there is an MLP of width $\leq C_p r$ computing $f_i$.

\end{proof}

\bibliography{bibliography}
\bibliographystyle{alpha}

\end{document}

%% file: importsanddeclares.tex
\usepackage{amsmath}
\usepackage{amsthm}
\usepackage{amssymb}
\usepackage{color}
\usepackage{mathtools}
\usepackage{makecell}
\usepackage{bm}
\usepackage{todonotes}
\usepackage{diagbox}
\usepackage{tablefootnote}
\usepackage{hyperref}
\usepackage[ruled,vlined,linesnumbered]{algorithm2e}
\usepackage{layouts}

\DeclareMathOperator{\rank}{rank}

\let\baraccent=\= %
\renewcommand{\=}[1]{\stackrel{#1}{=}} %

\providecommand{\cD}{\mathcal{D}}

\providecommand{\cS}{\mathcal{S}}

\renewcommand{\P}{\mathsf{P}}

\mathchardef\mhyphen="2D %

\providecommand{\sm}{\setminus}

\newcommand{\interior}[1]{%
  {\kern0pt#1}^{\mathrm{o}}%
}

\usepackage[capitalise]{cleveref}
\newtheorem{theorem}{Theorem}[section]
\newtheorem{lemma}[theorem]{Lemma}

\newtheorem{proposition}[theorem]{Proposition}

\newtheorem{hypothesis}[theorem]{Hypothesis}

\newtheorem{definition}[theorem]{Definition}

\newtheorem{remark}[theorem]{Remark}

\usepackage[utf8]{inputenc} %
\usepackage[T1]{fontenc}    %
\usepackage{hyperref}       %
\usepackage{url}            %
\usepackage{booktabs,comment}       %
\usepackage{amsfonts}       %
\usepackage{nicefrac}       %
\usepackage{microtype}      %
\usepackage{xcolor}         %
\usepackage{amsmath}
\usepackage{amssymb}
\usepackage{amsthm}
\usepackage{bbm}
\usepackage{xcolor}
\usepackage{graphicx}
\usepackage{subcaption}
\usepackage{adjustbox}
\usepackage{url}

\usepackage{hyperref}
\hypersetup{
    colorlinks,
    linkcolor={blue!80!black},
    citecolor={green!50!black},
}
\usepackage{xcolor}
\colorlet{linkequation}{blue}

\usepackage{comment}
\usepackage{latexsym}
\usepackage{amsmath}
\usepackage{amssymb}
\usepackage{amsthm}
\usepackage{epsfig}
\usepackage{graphicx}
\usepackage{bm}
\usepackage[shortlabels]{enumitem}
\usepackage{graphicx}
\usepackage{bbm}
\usepackage{accents}
\usepackage{mathtools}

\usepackage[top=1in, bottom=1in, left=1in, right=1in]{geometry}

\renewcommand{\P}{\mathbb{P}}
\newcommand{\E}{\mathbb{E}}

\newcommand{\R}{\mathbb{R}}

\newtheoremstyle{myremark} %
    {\topsep}                    %
    {\topsep}                    %
    {\rm}                        %
    {}                           %
    {\bf}                        %
    {.}                          %
    {.5em}                       %
    {}  %

\DeclareSymbolFont{rsfs}{U}{rsfs}{m}{n}
\DeclareSymbolFontAlphabet{\mathscrsfs}{rsfs}

\def\cS{{\mathcal S}}

\def\cD{{\mathcal D}}

\def\cS{{\mathcal S}}

\crefname{claim}{claim}{claims}
\crefname{fact}{fact}{facts}

\newtheorem{informaltheorem}[theorem]{Informal Theorem}

\usepackage{hyperref}
\usepackage{url}

\usepackage[many]{tcolorbox}

\usepackage{varwidth}%

\usepackage{hyperref}
\hypersetup{
    colorlinks,
    linkcolor={blue!80!black},
    citecolor={green!50!black},
}

\usepackage[top=1in, bottom=1in, left=1in, right=1in]{geometry}

\usepackage{amsmath}
\numberwithin{equation}{section}
\usepackage{amssymb}
\usepackage{mathtools}
\usepackage{amsthm}

\usepackage{thmtools}
\usepackage{thm-restate}

\usepackage[numbers, compress]{natbib}

\usepackage{tikz}
\usetikzlibrary{shapes, arrows, calc, positioning, arrows.meta}

\DeclareMathOperator\cov{cov}

\newcommand{\smax}{\mathrm{softmax}}

\usepackage{wrapfig}
\usepackage{sidecap}

\usepackage{fancybox}
\newenvironment{fminipage}%
  {\begin{Sbox}\begin{minipage}}%
  {\end{minipage}\end{Sbox}\fbox{\TheSbox}}

\usepackage{longtable}
\usepackage{booktabs}
\usepackage{soul}
\setuldepth{hi}